\documentclass[11pt]{article}
\usepackage{color}
\usepackage{pdfsync,enumitem,dsfont}

\usepackage[T1]{fontenc}
\usepackage[latin9]{inputenc}
\usepackage{lmodern}
\usepackage{geometry}
\usepackage{amsthm}
\usepackage{amsmath}
\usepackage{amssymb}
\usepackage{amsthm}
\usepackage{esint}
\usepackage[numbers, sort]{natbib}

\usepackage{subcaption, caption}
\usepackage{graphicx}
\usepackage{authblk}
\newcommand{\E}{\mathbb{E}}
\newcommand{\Prob}{\mathbb{P}}
\newcommand{\hist}{\mathcal{H}_{t-1}}
\newcommand{\regret}{{\rm Regret}}
\newcommand{\sregret}{{\rm SRegret}}
\newcommand{\Ac}{\mathcal{A}}
\newcommand{\Yc}{\mathcal{Y}}
\newcommand{\Xc}{\mathcal{X}}
\newcommand{\Rc}{\mathcal{R}}
\newcommand{\Hc}{\mathcal{H}}

\newcommand{\N}{\mathbb{N}_0}
\DeclareMathOperator*{\argmin}{argmin}
\DeclareMathOperator*{\argmax}{argmax}

\theoremstyle{plain}

\newtheorem{thm}{Theorem}
\newtheorem{lemma}[thm]{Lemma}

\newtheorem{example}{Example}
\newtheorem{fact}{Fact}

\newtheoremstyle{TheoremNum}
        {\topsep}{\topsep}              
        {\itshape}                      
        {}                              
        {\bfseries}                     
        {.}                             
        { }                             
        {\thmname{#1}\thmnote{ \bfseries #3}}
\theoremstyle{TheoremNum}

\date{}

\long\def\comment#1{}
\begin{document}

\title{Satisficing in Time-Sensitive Bandit Learning}
\author[1]{Daniel Russo}
\author[2]{Benjamin Van Roy}
\affil[1]{Columbia University, djr2174@columbia.edu}
\affil[2]{Stanford University, bvr@stanford.edu}

\maketitle
\begin{abstract}
		Much of the recent literature on bandit learning focuses on algorithms that aim to converge on an optimal action.  One shortcoming is that this orientation does not account for time sensitivity, which can play a crucial role when learning an optimal action requires much more information than near-optimal ones.  Indeed, popular approaches such as upper-confidence-bound methods and Thompson sampling can fare poorly in such situations.  We consider instead learning a {\it satisficing action}, which is near-optimal while requiring less information, and propose {\it satisficing Thompson sampling}, an algorithm that serves this purpose.  We establish a general bound on expected discounted regret and study the application of satisficing Thompson sampling to linear and infinite-armed bandits, demonstrating arbitrarily large benefits over Thompson sampling.  We also discuss the relation between the notion of satisficing and the theory of rate distortion, which offers guidance on the selection of satisficing actions.
\end{abstract}

\section{Introduction}\label{sec:intro}
As noted by Herbert Simon when receiving his Nobel Prize, ``decision makers can satisfice either by finding optimum solutions for a simplified world, or by finding satisfactory solutions for a more realistic world.''  Oriented around the former, much of the bandit learning literature has focused on algorithms that aim to converge on an optimal action.  As this area progresses, researchers study increasingly complex models, often with very large or infinite action sets.  For some such models, convergence to optimality can be a very slow process.  It is natural to ask whether a satisficing action can be identified much more quickly.  If so, that may be preferable, especially considering that the model itself is a simplified approximation to reality and thus even an exact optimal action is not truly optimal.  The following example illustrates the issue. \\
\begin{example}{\bf (Many-Armed Deterministic Bandit)}
	\label{ex:many-armed}
	Consider an action set $\Ac=\{1,\ldots,K\}$.  Each action $a \in \Ac$ results in reward $\theta_a$.  We refer to this as a {\it deterministic bandit} because the realized reward
	is determined by the action and not distorted by noise.  The agent begins with a prior over each $\theta_a$ that is independent and uniform over $[0,1]$
	and sequentially applies actions $A_0,A_1,A_2,\ldots$, selected by an algorithm that adapts decisions as rewards are observed.
	As $K$ grows, it takes longer to identify the optimal action $A^* = \argmax_{a \in \Ac} \theta_a$.  Indeed, for any algorithm, $\Prob(A^* \in \{A_0,\ldots A_{t-1}\}) \leq t/K$.
	Therefore, no algorithm can expect to select $A^*$ within time $t\ll K$.  On the other hand, by simply selecting actions in order, with $A_0 = 1, A_1 = 2, A_2 = 3, \ldots$,
	the agent can expect to identify an $\epsilon$-optimal action within $t = 1/\epsilon$ time periods, independent of $K$. 
\end{example}
\vspace{\baselineskip}

It is disconcerting that popular algorithms perform poorly when specialized to this simple problem.  Thompson sampling (TS) \cite{thompson1933}, 
for example, is likely to sample a new action
in each time period so long as $t \ll K$.  The underlying issue is most pronounced in the asymptotic regime of $K \to \infty$, 
for which TS never repeats any action because, at any point in time, there will be actions better than those previously selected. 
A surprisingly simple modification offers dramatic improvement: settle for the first action $a$ for which $\theta_a \geq 1-\epsilon$.
This alternative can be thought of as a variation of TS that aims to learn a {\it satisficing} action $\tilde{A} = \min\{a: \theta_a \geq 1-\epsilon\}$.  
We will refer to an algorithm that samples from the posterior
distribution of a satisficing action instead of the optimal action, as {\it satisficing Thompson sampling} (STS).

While stylized, the above example captures the essence of a basic dilemma faced in all decision problems and not adequately addressed by 
popular algorithms.  The underlying issue is time preference.  In particular, if an agent is only concerned about performance over an 
asymptotically long time horizon, it is reasonable to aim at learning $A^*$, while this can be a bad idea if shorter term performance matters
and a satisficing action $\tilde{A}$ can be learned more quickly.
To model time preference and formalize benefits of STS, we will assess performance in terms of expected discounted regret,
which for the many-armed deterministic bandit can be written as $\E\left[\sum_{t=0}^\infty \alpha^t (\theta_{A^*} - \theta_{A_t})\right]$.
The constant $\alpha \in [0,1]$ is a discount factor that conveys time preference.
It is easy to show, as is done through Theorem \ref{thm: regret of TS} in the appendix, that in the asymptotic regime of $K \to \infty$, TS experiences expected discounted regret of
$1/2(1-\alpha)$, whereas that of STS is bounded above by $1/\sqrt{1-\alpha}$.  For $\alpha$ close to $1$, we have $1/\sqrt{1-\alpha} \ll 1/(1-\alpha)$, and therefore STS vastly outperforms TS.
In fact, as $\alpha$ approaches $1$, the ratio between expected discounted regret of TS and that of STS goes to infinity.
This stylized example demonstrates potential advantages of STS.  

Of course, satisficing can also play a critical role in more complicated decision problems. The following example provides one simple illustration by adding a structured correlation pattern to the infinite-armed deterministic bandit treated in Example 1. \\
\begin{example}{\bf (Infinite-Armed Deterministic Bandit With Hierarchical Structure)}
	A website seeks the best advertisement to display out of an enormous number of alternatives. Each particular ad $a\in \mathcal{A}$ is associated with a known binary feature vector $\phi(a) \in \{0, 1\}^d$, where individual features may indicate whether an ad pertains to a particular category of product, contains an image of a particular celebrity, includes bright colors, etc.  When displayed, advertisement $a$ generates revenue governed by the linear mixed model
	\[
	R_a=\sum_{j=1}^{d} \phi(a)_j \theta_{0j} + \theta_{1a}  
	\] 
	where $(\theta_{01},\cdots, \theta_{0d})$ encodes the effect of each feature, while $\theta_{1a}$ encodes an ad-specific effect.  Any given feature vector $\phi(a)$ for $a\in \Ac$ is shared by an infinite number of other ads, i.e. $|\{ a'\in \Ac \mid   \phi(a') = \phi(a)\}| =\infty$. This serves to model settings in which there are many ads with common features relative to the problem's time horizon. The agent begins with a prior over $\theta=((\theta_{0j})_{j=1,\cdots d},   (\theta_{1a})_{a\in \Ac} )$ under which each variable is independent with marginal distributions $\theta_{0j} \sim N\left( 0, \sigma_0^2 \right)$ and $\theta_{1a} \sim {\rm Uniform}(0,1)$. 
\end{example}
\vspace{\baselineskip}
Similar to the infinite armed bandit with independent arms, identifying an exactly optimal arm is hopeless, though it may be possible to quickly identify a satisficing arm.  The search for a satisficing arm can be further accelerated because the linear mixed model enables some generalization across arms.  In particular, data gathered from trying some arms inform estimates for other arms, helping to guide the search.

In order to avoid trying each individual arm, one might be tempted to approximate this model via a linear bandit, ignoring the  ad specific effect $\theta_{1a}$ altogether.  In particular, one could assume that rewards are generated according to $\sum_{j=1}^{d} \phi(a)_j \theta_{0j} + W_t$, where $W_t$ is iid zero-mean noise.  But ignoring consistent ad-specific effects in this way can lead popular algorithms like Thompson sampling and linear upper-confidence bound approaches to converge on poorly performing arms.  We will revisit this example in Section \ref{sec: Hierarchical}, where we demonstrate how our approach can efficiently identify satisficing actions while modeling ad-specific effects.

\subsection{Our Contributions.}
This paper develops a general framework for studying satisficing in sequential learning.  Satisficing algorithms aim to learn a satisficing action $\tilde{A}$.  Building on the work of \citet{russo2016info}, we will establish a general information-theoretic regret bound, showing that any algorithm's expected discounted regret relative to a satisficing action is bounded in terms of the mutual information $I(\theta; \tilde{A})$ between the model parameters $\theta$ and the satisficing action and a newly-defined notion of information ratio, which measures the cost of information acquired about the satisficing action.  The mutual information $I(\theta; \tilde{A})$ can be thought of as the number of bits of information about $\theta$ required to identify $\tilde{A}$, and the fact that the bound depends on this quantity instead of the entropy of $A^*$, as does the bound of \citep{russo2016info}, allows it to capture the reduction of discounted regret made possible by settling for the satisficing action.

A natural and deep question concerns the the choice of satisficing action $\tilde{A}$ and the limits of performance attainable via satisficing. An exploration of this question yields novel connections between sequential learning and rate-distortion theory. In Section, \ref{sec: rate distortion} we define a natural rate-distortion function for Bayesian decision making, which captures the minimal information about $\theta$ a decision-maker must acquire in order to reach an $\epsilon$-optimal decision.  Combining this rate-distortion function with our general regret bound leads to new results and insights.  As an example, while previous information-theoretic regret bounds for the linear bandit problem become vacuous in contexts with infinite action spaces, our rate-distortion function leads to a strong bound on expected discounted regret.

We will also study the infinite-armed bandit problem with noisy rewards.  Here, we will consider a satisficing action $\tilde{A} = \min\{a: \theta_a \geq 1-\epsilon\}$. Simple numerical experiments demonstrate the benefits of STS over TS. We instantiate our general regret analysis in the infinite-armed bandit problem by bounding the mutual information and information ratio in that problem. This yields a bound on expected discounted regret that formalized the benefits of STS over TS. We complement this upper bound by establishing a matching lower bound on the regret of any algorithm for infinite armed bandit. 

\subsection{Alternative approaches.}

Many papers \citep{kleinberg2008multi, rusmevichientong2010linearly, bubeck2011xarmed}  have studied bandit problems with continuous action spaces, where it is also necessary to learn only approximately optimal actions. However, because these papers focus on the asymptotic growth rate of regret they implicitly emphasize later stages of learning, where the algorithm has already identified extremely high performing actions but exploration is needed to identify even better actions. Our discounted framework instead focuses on the initial cost of learning to attain good, but not perfect, performance.  Recent papers \citep{francetich2016toolkita,francetich2016toolkitb} study several heuristics for a discounted objective, though without an orientation toward formal regret analysis.  The knowledge gradient algorithm of \citet{ryzhov2012knowledge}  also takes time horizon into account and can learn suboptimal actions when its not worthwhile to identify the optimal action. This algorithm tries to directly approximate the optimal Bayesian policy using a one-step lookahead heuristic, but there are no performance guarantees for this method. \citet{deshpande2012linear} consider a linear bandit problem with dimension that is too large relative to the desired horizon. They propose an algorithm specifically for that problem that limits exploration and learns something useful within this short time frame.

\citet{berry1997bandit, wang2009algorithms} and \citet{bonald2013two} study an infinite-armed bandit problem in which it is impossible to identify an optimal action and propose algorithms to minimizes the asymptotic growth rate of regret. 
Their strategies are carefully designed, but appear to be difficult to adapt to more complex problems. For example, one algorithm in \citet{berry1997bandit} discards an arm as soon as it produces a reward of zero in some period, which is sensible only for infinite armed bandits with uniform prior and binary rewards. A brief section describes extensions to non-uniform priors, but still these cannot extend beyond binary feedback. \citet{bonald2013two} design a procedure specifically for the infinite-armed bandit with uniform prior and binary rewards under which the first-order term in an asymptotic expansion of regret is minimized.  The algorithm of \citet{wang2009algorithms} discards all but $k$ arms at the start of the problem then applies a standard algorithm for $k$--armed bandits. The main contribution of that work is to carefully analyze regret as a function of the prior, allowing them to choose $k$ to minimize regret upper bounds. While we will instantiate our general regret bound for STS on the infinite-armed bandit problem, we view this example as a very simple and stylized special case that we provide to illustrate basic concepts.  The flexibility of STS and our analysis framework allow this work to be applied to more complicated time-sensitive learning problems. 

Our approach is built on Thompson sampling. One simple reason is that Thompson sampling is enjoying wide practical use due to its ease of use, ability to incorporate complex prior information, and resilience to delayed feedback \cite{russo2018tutorial, scott2010modern, chapelle2011empirical}. Given this, we see value in broadening the class of problem to which Thompson sampling approaches may be applied. It is also worth emphasizing that formulations of problems like the infinite armed-bandit in Example 1 are inherently Bayesian. Arms are modeled as independent and yet the decision-maker is required to make inferences about the quality of arms for which no data is available. By contrast, frequentist algorithms like KL-UCB \cite{KL-UCB2013} usually require an initial phase in which all arms are tested at least once. Beyond simple infinite-armed bandits, developing satisficing variants of Thompson sampling appears to be a natural way to approach more complex problems like the hierarchical bandit in Example 2 -- which require both satisficing and generalizing across arms. A separate motivation for us is to improve the information theoretic analysis of complex Bayesian bandit problems. This analysis is centered on Thompson sampling, but beyond any interest in the algorithm it has been an effective tool for establishing regret upper bounds, including for bandit convex optimization \cite{bubeck2015multi}, for partial monitoring \cite{lattimore19Info}, for bandits with graph-structured feedback \cite{tossou2017thompson}, and for reinforcement learning \cite{lu2019information}.

\section{Problem Formulation.}\label{sec: formulation}
An agent sequentially chooses actions $(A_t)_{t\in \N}$ from the action set $\Ac$ and observes the corresponding outcomes $\left(Y_{t}\right)_{t\in \N} \subset \mathcal{Y}$. The agent associates a reward $R(y)$ with each outcome $y\in \Yc$. Let $R_{t} \equiv R(Y_{t})$ denote the reward corresponding to outcome $Y_{t}$.
The outcome $Y_t$ in period $t$ depends on the chosen action $A_t$, idiosyncratic randomness associated with that time step, and a random variable $\theta$ that is fixed over time. Formally, there is a known system function $g$, and an iid sequence of disturbances $(W_t)_{t\in \N}$ such that
\[ 
Y_t = g(A_t, \theta, W_t).
\] 
The disturbances $W_t$ are independent of $\theta$, and have a known distribution. This is without loss of generality, as uncertainty about $g$ and the distribution of $W_t$ could be included in the definition of $\theta$. From this, we can define
\[ 
\mu(a, \theta)= \E[R\left(g(a, \theta, W_t)\right) | \theta]
\] 
to be the expected reward of an action $a$ under parameter $\theta$. Ours can be thought of as a Bayesian formulation, in which the distribution of $\theta$ represents the agent's prior uncertainty about the true characteristics of the system, and conditioned on $\theta$, the remaining randomness in $Y_t$ represents idiosyncratic noise in observed outcomes. 

The history available when selecting action $A_t$ is $\hist = (A_0, Y_{0}, \ldots, A_{t-1}, Y_{t-1})$. The agent selects actions according to a policy, which is a sequence of functions $\psi=(\psi_t)_{t \in \N}$, each mapping a history and an exogenous random variable $\xi$ to an action, with $A_t= \psi_{t}(\hist, \xi)$ for each $t$. Throughout the paper, we use $\xi$ to denote some random variable that is independent of $\theta$ and the disturbances $(W_{t})_{t\in \N}$. 




Let $R^*  = \sup_{a\in \Ac} \mu(a, \theta)$ denote the supremal reward, and let  $A^* \in \,\, \argmax_{a \in \Ac} \, \mu(a,\theta)$ denote the true optimal action, when this maximum exists.  As a performance metric, we consider \emph{expected discounted regret} of a policy $\psi$ is defined by
\[
\regret (\alpha, \psi) =\E^{\psi}\left[\sum_{t=0}^{\infty} \alpha^{t} (R^* - R_{t})\right],
\]
which measures a discounted sum of the expected performance gap between an omniscient policy which always chooses the optimal action $A^*$ and the policy $\psi$ which selects the actions $(A_t)_{t\in \N}$. This deviates from the typical notion of expected regret in its dependence on a discount factor $\alpha \in [0,1]$.
Regular expected regret corresponds to the case of $\alpha = 1$.  Smaller values of $\alpha$ convey time preference by weighting gaps in nearer-term performance higher than gaps in longer-term performance.

The definition above compares regret relative to the optimal action $A^*$ and corresponding reward $R^*$. It is useful to also consider performance loss relative to a less stringent benchmark. We define the satisficing regret at level $D\geq 0$ to be 
\[
\sregret (\alpha, \psi, D) =\E^{\psi}\left[\sum_{t=0}^{\infty} \alpha^{t} (R^*-D - R_{t})\right].
\]
This measures regret relative to an action that is near-optimal, in the sense that it yields expected reward $R^*-D$, which is within $D$ of optimal. This notation was chosen due to the connection we develop with rate-distortion theory, where $D$ typically denotes a tolerable level of ``distortion'' in a lossy compression scheme. Of course, for all $D\geq 0$, 
\begin{equation}\label{eq: sregret to regret}
\regret(\alpha, \psi) = \sregret (\alpha, \psi, D) + \frac{D}{1-\alpha}
\end{equation}
and so one can easily translate between bounds on regret and bounds on satisficing regret. However, directly studying satisficing regret helps focus our attention on the design of algorithms that purposefully avoid the search for exactly optimal behavior in order to limit exploration costs. 

\subsection*{Additional Notation.}
Before beginning, let us first introduce some additional notation.
We denote the entropy of a random variable $X$ by $H(X)$, the Kullback-Leibler divergence between probability distributions $P$ and $Q$ by $D(P||Q)$, and the mutual information between two random variables $X$ and $Y$ by $I(X; Y)$. We will frequently be interested in the conditional mutual information $I(X; Y| \hist)$. 

We sometimes denote by $\E_{t}[\cdot]=\E[\cdot | \hist ]$ the expectation operator conditioned on the history up to time $t$ and similarly define $\Prob_{t}(\cdot) = \Prob(\cdot | \hist)$. The definitions of entropy and mutual information depend on a base measure. We use $H_{t}(\cdot)$ and $I_{t}(\cdot\, ,\cdot)$ to denote entropy and mutual-information when the base-measure is the posterior distribution $\Prob_{t}$. For example, if $X$ is and $Z$ are discrete random variables taking values in a sets $\mathcal{X}$ and $\mathcal{Z}$,
\[
I_{t}(X; Z) =  \sum_{x\in \mathcal{X}}\sum_{z\in \mathcal{Z}} \Prob_{t}(X=x, Z=z) \log\left(\frac{\Prob_{t}(X=x, Z=z)}{\Prob_{t}(X=x)\Prob_{t}(Z=z)}\right).
\]
Due to its dependence on the realized history $\hist$, $I_{t}(X;Z)$ is a random variable. The standard definition of conditional mutual information integrates over this randomness, and in particular, $\E[I_{t}(X;Z)] = I(X; Z | \hist)$.

\section{Satisficing Actions.}\label{sec: satisficing}

We will consider learning a satisficing action $\tilde{A}$ instead of an optimal action $A^*$.  The idea is to target a satisficing action that is near-optimal yet easy to learn.  The information about $\theta$ required to learn an action $\tilde{A}$ is captured by $I(\theta; \tilde{A})$, while the performance loss is $\E[R^* - \tilde{R}]$, where $\tilde{R} = \mu(\tilde{A}, \theta)$.  For $\tilde{A}$ to be easy to learn relative to $A^*$, we want $I(\theta; \tilde{A}) \ll I(\theta; A^*)$.  We will motivate this abstract notion through several examples.

Our first example addresses the infinite-armed deterministic bandit, as discussed in Section \ref{sec:intro}..
\begin{example}[first satisfactory arm]
	Consider the infinite-armed deterministic bandit of Section \ref{sec:intro}.  For this problem, the prior of $A^*$ is uniformly distributed across a large number $K$ of actions, and $I(\theta; A^*) = H(A^*) = \log K$.  Consider a satisficing action $\tilde{A} = \min\{ k | \theta_{k} \geq R^*-\epsilon\}$, 
	which represents the first action that attains reward within $\epsilon$ of the optimum $R^*=\max_{k} \theta_k$. As $K \to \infty$, $I(\theta; A^*) \to \infty$. But $I(\theta; \tilde{A})$ remains finite, as in this limit $\tilde{A}$ converges weakly to a geometric random variable, with $I(\theta; \tilde{A}) = H(\tilde{A}) = -((1-\epsilon) \log(1-\epsilon) + \epsilon \log(\epsilon))/\epsilon$.  
\end{example}

The next example addresses the infinite-armed bandit with hierarchical structure from Section \ref{sec:intro}. Naturally, the satisficing action is itself defined in a hierarchical way, first defining a subset of arms with the most attractive features and then selecting the first arm among those with a satisfactory arm-specific effect. 
\begin{example}[hierarchical satisficing]\label{ex: hiearchical satisficing action}
	Consider a hierarchical infinite-armed deterministic bandit of Section \ref{sec:intro}. The set of arms is $\Ac=\{1, 2, 3, \ldots\}$. The reward generated by an arm $a\in \Ac$ can be written as $\mu(a, \theta)  = \langle \phi(a) \, , \, \theta_0 \rangle + \theta_{1a}$, where $\phi(a) \in \{0,1  \}^d$ is a known feature vector associated with the arm, $\theta_0 \in \mathbb{R}^d$ is drawn from some prior and the arm-specific effects $\theta_{1,a} \sim {\rm Unif}[0,1]$ are drawn independently across arms $a\in \Ac$. For each action $a$, we assume there are is infinite collection of other actions $\{a' \mid \phi(a') = \phi(a)\}$ that share the same feature vector. (One should have in mind problems where the features encode relatively coarse categories.) We take $\Ac(\theta_0) =\{a\in \Ac  \mid  \langle \phi(a), \theta_0 \rangle  \geq \langle \phi(a')\,,\, \theta_0 \rangle \quad \forall a' \in \Ac \}$  to be the subset of actions that have optimal features. We take a satisficing action to be the first such action that has a sufficiently large ad-specific effect. 
	\begin{equation}\label{eq: hierarchical satisficing} 
	\tilde{A} = \argmin_{ a \in \mathcal{A}(\theta_0)}   \{a|  \theta'_a \geq  1-  \epsilon \}. 
	\end{equation}
	In this case, although there are an infinite number of arms, the entropy of the satisficing action is bounded as: 
	\begin{align*}
	I(\tilde{A}; \theta) = I(\tilde{A}; \theta_0) + I(\tilde{A}; \theta_1  \mid \theta_0)  \leq d\log(2)-((1-\epsilon) \log(1-\epsilon) + \epsilon \log(\epsilon))/\epsilon.
	\end{align*}	
	where we have used that $\theta_0$ is supported on $2^d$ possible values and that, conditioned on $\theta_0$, $\tilde{A}$ follows a geometric distribution with survival probability $1-\epsilon$. 
\end{example}

The next example involves reducing the granularity of a discretized action space.
\begin{example}[discretization] 
	Consider a linear bandit with $\Ac \subset \mathbb{R}^p$ and $\mu(a, \theta) = a^\top \theta$  for an unknown vector $\theta$. Suppose that $\theta \sim N(0,I)$ and $\Ac$ consists of $K$ vectors spread out uniformly along boundary of the $d$-dimensional unit sphere $\{a \in \mathbb{R}^p : \| a\|_2 =1\}$. The optimal action $A^* = \argmax_{a\in \Ac} a^\top \theta$  is then uniformly distributed over $\Ac$, and therefore $I(\theta; A^*) = H(A^*) = \log K$.  As $K \to \infty$, it takes an enormous amount of information to exactly identify $A^*$. The results of \cite{russo2014learning} become vacuous in this limit.  Consider a satisficing action $\tilde{A}$ that represents a coarser version of $A^*$. In particular, for $M \ll K$, let $\tilde{\Ac}$ consist of $M$ vectors spread out uniformly along boundary of the $d$-dimensional unit sphere, with $M$ chosen such that for each element of $\Ac$ there is a close approximation in $\tilde{\Ac}$.
	Let $\tilde{A} = \argmax_{a\in \tilde{\Ac}} \theta^\top a$.  This can be viewed as a form of lossy-compression, for which $H(\tilde{A}) \ll H(A^*)$ while $\E[R^* - \tilde{R}]$ remains small.  
\end{example}
In the previous examples, $\theta$ determined $\tilde{A}$, and therefore $I(\theta; \tilde{A}) = H(\tilde{A})$.  We now consider an example in which $I(\theta; \tilde{A})$ is controlled by randomly perturbing the satisficing action.  Here, $I(\theta; \tilde{A})$ can be small even though $H(\tilde{A})$ is large.
\begin{example}[random perturbation]
	Consider again the linear bandit from the previous example. An alternative satistficing action $\tilde{A}$ results from optimizing a perturbed objective $\tilde{A} \in \argmax_{a \in \Ac} a^T (\theta + Z)$ where $Z\sim N(0, (\epsilon/p)^2 I)$. Since $Z$ is not observable, it is not possible in this case to literally learn $\tilde{A}$.  Instead, we consider learning to behave in a manner indistinguishable from $\tilde{A}$.  The variance of $Z$ is chosen such that $\E[R^* - \tilde{R}] = \epsilon$. Moreover, it can be shown that $I(\tilde{A}; \theta)\leq I(\theta+Z; \theta)  = p\log(1+p^2/\epsilon^2)$ and therefore, the information required about $\theta$ is bounded independently of the number of actions.   
\end{example}

	\section{A General Regret Bound.}
This section provides a general discounted regret bound and a new information-theoretic analysis technique. The first subsection introduces an alternative to the information ratio of \citet{russo2016info}, which is more appropriate for time-sensitive online learning problems. The following subsection establishes a general discounted regret bound in terms of this information ratio. 


\subsection{A New Information Ratio.}

First, we make a simplification to the information ratio $(\E_{t}[R^* - R_{t}])^2 / I_{t}(A^* ; (A_t, Y_{t}))$ defined by \citet{russo2016info}. That expression depends on the history $\hist$ and hence is a random variable. In this paper, we observe that this can be avoided, and instead take as a starting point a simplified form of the information ratio that integrates out all randomness. In particular, we study 
\begin{equation}\label{eq: old-info-ratio in expectation}
\frac{(\E[R^* - R_{t}])^2}{I( A^* ; (A_t, Y_{t}) \mid \hist)}.
\end{equation}
Uniform bounds on the information ratio of the type established in past work \cite{bubeck2015multi,russo2016info, liu2017information} imply those on \eqref{eq: old-info-ratio in expectation}. Precisely, if $(\E_{t}[R^* - R_{t}])^2 / I_{t}(A^* ; (A_t, Y_{t}))$ is bounded by $\lambda \in \mathbb{R}$ almost surely (i.e for any history $\hist$), then \eqref{eq: old-info-ratio in expectation} is bounded by $\lambda$ since
\[ 
(\E[R^* - R_{t}])^2 \leq \E[(\E_{t}[R^* - R_{t}])^2]  \leq   \E[\lambda I_{t}(A^* ; (A_t, Y_{t}))] = \lambda I(A^*; (A_t, Y_{t}) | \hist).
\]
A more important change comes from measuring information about a benchmark action $\tilde{A}$, which could be defined as in the  examples in the previous section, rather than with respect to the optimal action $A^*$. For a benchmark action $\tilde{A}$ we consider the single period information ratio
\[
\frac{(\E[\tilde{R} - R_{t}])^2}{I( \tilde{A} ; (A_t, Y_{t}) \mid \hist)}
\]
where $\tilde{R} = \mu(\tilde{A}, \theta)$. This ratio relates the current shortfall in performance relative to the benchmark action $\tilde{A}$ to the amount of information acquired about the benchmark action. We study the discounted average of these single period information ratios, defined for any policy $\psi$ as
\begin{equation}\label{eq: information-ratio}
\Gamma\left(\tilde{A}, \psi \right) = (1-\alpha^2) \sum_{t=0}^\infty \alpha^{2t} \left(\frac{(\E[\tilde{R} - R_{t}])^2}{I(\tilde{A};  (A_t, Y_{t}) | \hist )}\right),
\end{equation}
where the actions $(A_t)_{t\in \N}$ are chosen under $\psi$.   The square in the discount factor $\alpha$ is consistent with the problem's original discount rate, since $(\E[\alpha^t(\tilde{R} - R_{t})])^2 = \alpha^{2t} (\E[\tilde{R} - R_{t}])^2$.

\subsection{General Regret Bound.}
The following theorem bounds the expected discounted regret of any algorithm, or policy, $\psi$ in terms of the information ratio \eqref{eq: information-ratio}.
\begin{thm}
	\label{th:discounted-regret}
	For any policy $\psi$, any $D\geq 0$ and any $\tilde{A}=f(\theta, \xi)$ where $\xi$ is independent of the disturbances $(W_t)_{t\in \N}$, if $\E[\mu(\tilde{A}, \theta)] \geq R^* -D$, then 
	\begin{equation*}
	\sregret(\alpha, \psi, D) \leq  \sqrt{\frac{\Gamma\left(\tilde{A}, \psi \right) I(\tilde{A}; \theta)}{1-\alpha^2}}.
	\end{equation*}
\end{thm}
\begin{proof}
	We first show that the mutual information between $\tilde{A}$ and $\theta$ bounds the expected accumulation of mutual-information between $\tilde{A}$ and observations $(A_t, Y_{t})_{t\in \N}$. By the chain rule for mutual information, for any $T$,
\begin{eqnarray*}
	\sum_{t=0}^{T} I(\tilde{A}; (A_t, Y_{t}) \mid \hist)
	&=& \sum_{t=0}^{T} I(\tilde{A}; (A_t, Y_{t}) \mid A_0, Y_{0}, \ldots, A_{t-1}, Y_{t-1}) \\
	&=& I(\tilde{A}; \mathcal{H}_T)\\
	&\leq& I(\tilde{A}; (\theta,\mathcal{H}_{T})) \\
	&=& I(\tilde{A};\theta) + I(\tilde{A} ; \mathcal{F}_{T} | \theta) \\
	&=& I(\tilde{A};\theta)
\end{eqnarray*}
where the final inequality uses that, conditioned on $\theta$, $\tilde{A}$ is independent of $\mathcal{H}_T$. Taking the limit as $T \rightarrow \infty$ implies
\[
\sum_{t=0}^{\infty} I(\tilde{A}; (A_t, Y_{t}) \mid \hist) \leq I(\tilde{A};\theta),
\]
where, the infinite series is assured to converge by the non-negativity of mutual information. Now, let
\[
\Gamma_{t} \equiv \frac{(\E[\tilde{R} - R_{t}])^2}{I(\tilde{A}; (A_t,Y_{t}) | \hist)}
\]
denote the information ratio at time $t$ under the benchmark action $\tilde{A}$ and actions $\{A_t : t \in \N\}$ chosen according to $\psi$. Then 
\begin{eqnarray*}
	\sregret(\alpha, \psi, D) = \E\left[\sum_{t=0}^\infty \alpha^t (R^*-D - R_{t})\right] 
	&=& \sum_{t=0}^\infty \alpha^t \E\left[\tilde{R} - R_{t, A_t}\right] \\
	&=& \sum_{t=0}^\infty \sqrt{\alpha^{2t} \Gamma_{t}} \sqrt{I(\tilde{A}; (A_t,Y_{t}) | \hist)}\\
	&\leq&  \sqrt{\sum_{t=0}^\infty \alpha^{2t} \Gamma_{t}} \sqrt{\sum_{t=0}^\infty I(\tilde{A}; (A_t,Y_{t}) | \hist)} \\
	&\leq&  \sqrt{\left[\sum_{t=0}^\infty \alpha^{2t} \Gamma_{t}\right]}\sqrt{I(\tilde{A} ; \theta )} \\
	&=& \sqrt{\frac{\Gamma\left(\tilde{A}, \psi \right) I(\tilde{A}; \theta)}{1-\alpha^2}},
\end{eqnarray*}
where the first inequality follows from the Cauchy-Schwarz inequality and the second was established earlier in this proof.
\end{proof} 

An immediate consequence of this bound on satisficing regret is the discounted regret bound
\begin{equation}\label{eq: discounted-regret}
\regret(\alpha, \psi) \leq \frac{\E[R^*- \tilde{R}]  }{1-\alpha}+\sqrt{\frac{\Gamma\left(\tilde{A}, \psi \right) I(\tilde{A}; \theta)}{1-\alpha^2}}.
\end{equation}
This bound decomposes regret into the sum of two terms; one which captures the discounted  performance shortfall of the benchmark action $\tilde{A}$ relative to $A^*$, and one which bounds the additional regret incurred while learning to identify $\tilde{A}$. Breaking things down further, the mutual information $I(\theta; \tilde{A})$ measures how much information the decision-maker must acquire in order to implement the action $\tilde{A}$, and the information ratio measures the regret incurred in gathering this information. It is worth highlighting that for any given action process, this bound holds simultaneously for all possible choices of $\tilde{A}$, and in particular, it holds for the $\tilde{A}$ minimizing the right hand side of \eqref{eq: discounted-regret}.

\section{Connections With Rate Distortion Theory.}\label{sec: rate distortion}
This section considers the optimal choice of satisfactory action $\tilde{A}$ and develops connections with the theory of rate-distortion in information theory. We construct a natural rate-distortion function for Bayesian decision making in the next subsection. Subsection \ref{subsec: uniformly bounded info ratios} then develops a general regret bound that depends on this rate-distortion function.

\subsection{A Rate Distortion Function for Bayesian Decision Making.}\label{subsec: a rate distortion function}
In information-theory, the entropy of a source characterizes the length of an optimal lossless encoding. The celebrated rate-distortion theory  \citep[Chapter~10]{cover2012elements} characterizes the number of bits required for an encoding to be close in some loss metric. This theory resolves when it is possible to to derive a satisfactory lossy compression scheme while transmitting far less information than required for a lossless compression. The rate-distortion function for a random variable $X$ with domain $\mathcal{X}$ with respect to a loss function $\ell: \hat{\mathcal{X}} \times \mathcal{X} \to \mathbb{R}$ is 
\begin{eqnarray} \label{eq: general rate distortion definition}
\mathcal{R}(D) =& \min & I(\hat{X}; X) \\ \nonumber
&s.t.& \E[\ell(\hat{X},X)] \leq D
\end{eqnarray}
where the minimum is taken the choice of random variables $\hat{X}$ with domain $\hat{\Xc}$, and $I(X; \hat{X})$ denotes the mutual information between $X$ and $\hat{X}$. One can view this optimization problem as specifying a conditional distribution $P(\hat{X} \in \cdot | X)$ that minimizes the information $\hat{X}$ uses about $X$ among all choices incurring average loss less than $D$. 

We will explore a powerful link with sequential Bayesian decision-making, where the rate-distortion function characterizes the minimal amount of new information the decision--maker must gather in order make a satisfactory decision.  Typically \eqref{eq: general rate distortion definition} is applied in the context of representing $X$ as closely as possible by $\hat{X}$, and the loss function is taken to be something like the squared distance or total variation distance between the two. For our purposes, we replace $X$ with $\theta$, and $\hat{X}$ with a benchmark action $\tilde{A}$. The interpretation is that $\tilde{A}=f(\theta; \xi)$ is a function of the unknown parameter $\theta$ and some exogenous randomness $\xi$ that offers a similar reward to playing $A^*$ but hopefully can be identified using much less information about $\theta$.  We specify a loss function $\ell : \mathcal{A} \times \Theta \to \mathbb{R}$ measuring the single period regret from playing $a$ under $\theta$:
\[ 
\ell(a, \theta )= \max_{a' \in \mathcal{A}} \mu(a', \theta) - \mu(a, \theta). 
\]
As a result 
\[ 
\E[\ell(\tilde{A}, \theta )]= \E[R^*-\mu(\tilde{A}, \theta)].
\]
We come to the rate-distortion function
\begin{eqnarray}\label{eq: rate distortion function for decision making}
\Rc(D) :=  & \min & I(\tilde{A}; \theta) \\\nonumber
&s.t.& \E[R^*-\mu(\tilde{A}, \theta)] \leq D.
\end{eqnarray}
As before, the minimum above is taken over the choice of random variable $\tilde{A}$ taking values in $\mathcal{A}$. That is, the minimum is taken over all conditional probability distributions $\Prob(\tilde{A}\in \cdot | \theta)$ specifying a distribution over actions as a function of $\theta$. Since the choice $\tilde{A}=A^*$ is always feasible, for all $D>0$
\[ 
\Rc(D)\leq I(A^*; \theta)= H(A^*)
\]
where $H(A^*)$ denotes the entropy of the optimal action. Rate distortion is never larger than entropy, but it may be small even if the entropy of $A^*$ is infinite.  

The following, somewhat artificial, example explicitly links communication with decision-making and may help clarify the role of the rate-distortion function $\mathcal{R}(D)$.
\begin{example}
	A military command center waits to hear from an outpost before issuing orders. The outpost, stationed close to the conflict, determines its message based on a wealth of nuanced information -- at the level of readouts from weather sensors and full transcripts of intercepted enemy communication. The command post could make very complicated decisions as a function of the detailed information it receives, with the possibility of specifying commands at the level of individual troops and equipment. How much must decision quality degrade if decisions are based only on coarser information? At an intuitive level, the outpost only needs to communicate surprising revelations that are important to reaching a satisfactory decision. As a result, our answer can depend in a complicated way on the extent to which the outpost's observations are predictable apriori and the extent to which decision quality is reliant on this information. The rate-distortion function precisely quantifies these effects.
	
	To map this problem onto our formulation of the rate-distortion function, take $\theta$ to consist of all information observed by the outpost, $\tilde{A}$ to be the order issued by the command center, and the rewards to indicate whether the orders led to a successful outcome. The mutual information $I(\tilde{A}; \theta)$ captures the average amount of information the outpost must send in order for $\tilde{A}$ to be implemented. The goal is to develop a plan for placing orders that requires minimal communication from the outpost among all plans that degrade the chance of success by no more than $D$.
\end{example}

\subsection{Uniformly Bounded Information Ratios.}\label{subsec: uniformly bounded info ratios}
The general regret bound in Theorem \ref{th:discounted-regret} has a superficial relationship to the rate-distortion function $\mathcal{R}(D)$ through its dependence on the mutual information $I(\tilde{A}; \theta)$ between the benchmark action and the true parameter $\theta$. Indeed, for a benchmark action $\tilde{A}$ attaining the rate-distortion limit, $I(\tilde{A}; \theta)=\mathcal{R}(D)$ and we attain a regret bound that depends explicitly on the rate-distortion level. However, the information ratio $\Gamma\left(\tilde{A}, \psi \right)$ also depends on the choice of benchmark action, and may be infinite for a poor choice.

This second dependence on $\tilde{A}$ does not appear in rate-distortion theory, and reflects a fundamental distinction between communication problems and sequential learning problems. Indeed, a key feature enabling the sharp results of rate distortion theory is that no bit of information is more costly to send and receive than others; the question is to what extent useful communication is possible while many fewer bits of information on average. By contrast, sequential learning agents must explore to uncover information and the cost per unit of information uncovered may vary widely depending on which pieces of information are sought. This is accounted for by the information ratio $\Gamma\left(\tilde{A}, \psi \right)$, which roughly captures the expected cost, in terms of regret incurred, per bit of information acquired about the benchmark action.  

Despite this, regret bounds in terms of rate-distortion apply in many important cases. Theorem \ref{th:discounted-regret}, which is an immediate consequence of Theorem \ref{th:discounted-regret}, provides a general bound of this form. Roughly, the uniform information ratio  $\Gamma_{U}$ in the theorem reflects something about quality of the feedback the agent receives when exploring;  it means that for \emph{any} choice of benchmark action $\tilde{A}$ there is a sequential learning strategy that learns about $\tilde{A}$ with cost per bit of information less than $\Gamma_{U}$. The next section applies this result to online linear optimization,  where several possible uniform information ratio bounds are possible depending on the problems precise feedback structure. 

\begin{thm}
	Suppose there is a uniform bound on the information ratio
	\[
	\Gamma_{U}:=\sup_{\tilde{A}} \inf_{\psi} \Gamma\left(\tilde{A}, \psi \right)  <\infty. 
	\]	
	Then, for any $D\geq 0$ there exists a policy $\psi$ under which  
	\[
	\sregret(\alpha, \psi, D) \leq \sqrt{\frac{\Gamma_U \Rc(D)}{1-\alpha^2}}. 
	\]
\end{thm}

	\section{Application to Online Linear Optimization.}\label{sec: linear optimization}

Consider a special case of our formulation: the problem of learning to solve a linear optimization problem. Precisely, suppose expected rewards follow the linear model $\E[R_{t}|\theta, A_t]= \theta^\top A_t$ where $A_t\in \Ac \subset \mathbb{R}^p$, $\theta \in \mathbb{R}^p$, and $R_{t}\in \left[-\frac{1}{2},\frac{1}{2}\right]$ almost surely. We will consider several natural forms a feedback $Y_{t}$ the decision-maker may receive.

In each case, uniform bounds on the information ratio hold for satisficing Thompson sampling. More precisely, for any $\tilde{A}$ let  $\psi^{\rm STS}_{\tilde{A}}$ denote the strategy that randomly samples an action at each time $t$ by probability matching with respect to $\tilde{A}$, i.e. $\Prob(A_t \in \cdot | \hist) = \Prob(\tilde{A} \in \cdot | \hist)$. Applying the same proofs as in \cite{russo2016info} yields bounds of the form $\Gamma(\tilde{A}; \psi^{\rm STS}_{\tilde{A}}) \leq \lambda$, where $\lambda$ 
depends on the problem's feedback structure but not the choice of benchmark action. Now, let us choose $\tilde{A}$ to attain the rate distortion limit  \eqref{eq: rate distortion function for decision making}, so $I(\tilde{A}; \theta)=\Rc(D)$ and $\E[R^*-\mu(\tilde{A}; \theta) ] \leq D$. We denote by $\psi_D^{\rm STS}$ satisficing Thompson sampling with respect to this choice of satisfactory action.

\noindent{\bf Full Information.} Suppose $R_{t} =A_t^\top Z_t$ for a random vector $Z_t$ with $\E[Z_t | \theta, \hist] = \theta$. This is an extreme point of our formulation, where all information is revealed without active exploration. For all $\tilde{A}$, the information ratio is bounded as $\Gamma(\psi_{\tilde{A}}^{\rm STS}; \tilde{A})\leq 1/2$ and hence 
\[ 
\sregret(\alpha, \psi_D^{\rm STS}, D) \leq \sqrt{\frac{ \mathcal{R}(D)}{2(1-\alpha^2)}}.
\]	
\noindent{\bf Bandit Feedback.} Suppose the agent only observes the reward the action she chooses ($Y_{t}=R_{t}$). This is the so-called linear bandit problem. For all $\tilde{A}$, the information ratio is bounded as $\Gamma(\psi_{\tilde{A}}^{\rm STS}; \tilde{A})\leq p/2$.  This gives the regret bound
\[ 
\sregret(\alpha, \psi_D^{\rm STS}, D) 
\leq  \sqrt{\frac{ \mathcal{R}(D)p}{2(1-\alpha^2)}}.
\]
\noindent {\bf Semi-Bandit Feedback.} Assume again that $R_{t} =A_t^\top Z_t$ for all $a$. Take the action set $\Ac \subset \{0,1\}^p$ to consist of binary vectors where $\sum_{i=1}^{p} a_i \leq m$ for all $a\in \Ac$. Upon playing action $A_t=a$, the agent observes $Z_{t,i}$ for every component $\{i\in \{1,...,m\} :a_i =1 \} $ that was active in $a$. We make the additional assumption that the components of $Z_t$ are independent conditioned on $\hist$. Then, for all $\tilde{A}$, the information ratio is bounded as $\Gamma(\psi_{\tilde{A}}^{\rm STS}; \tilde{A})\leq p/2m$ and hence
\[ 
\sregret(\alpha, \psi^{\rm STS}_{D}, D) 
\leq  \sqrt{\frac{ \mathcal{R}(D)(p/m)}{2(1-\alpha^2)}}.
\]
By following the appendix of \cite{russo2016info}, each of these result can be extended gracefully to settings where noise distributions are \emph{sub-Gaussian}. For example, suppose $\theta$ follows a multivariate Gaussian distribution, and the reward at time $t$ is $R_t = \theta^\top A_t + W_t$ where $W_t$ is a zero mean Gaussian random variable. Then, if the variance of rewards $\E[(\theta^\top a + W-\E[\theta]^\top a)^2]\leq \sigma^2$
is bounded by some $\sigma^2$ for all $a$, the previous bounds on the information ratio scale by a factor of $\sigma$.

It is worth noting that these results immediately reduce to bounds on (non-satisficing) regret when the action space is finite. As mentioned in Subsection \ref{subsec: a rate distortion function}, for problems with a finite action set $\Rc(D) \leq H(A^*)\leq \log | \Ac|$ for all $D\geq 0$. For example, with a linear bandit with finite action set, \eqref{eq: sregret to regret} gives the bound 
\[
\regret( \alpha, \psi^{\rm STS}_{D}) \leq  \sqrt{\frac{ H(A^*) p}{2(1-\alpha^2)}} + \frac{D}{1-\alpha}, 
\]
so STS with a small satisficing level attains desirable regret bounds when the action set is not too large. A special case of this problem is the classical $k$ armed bandit problem, in which case $p=k$. We can reach similar conclusions in other cases by arguing $\Rc(D)$ grows slowly as $D\to 0$. We illustrate this idea for a Gaussian linear bandit below.

The next theorem considers an explicit choice of satisfactory action $\tilde{A}$. This yields a computationally efficient version of STS as well as explicit upper bounds on the rate-distortion function. As above, consider the case where $\theta\sim N(\mu, \Sigma)$ follows a multivariate Gaussian prior and reward noise is Gaussian. We study the optimizer $\tilde{A}=\argmax_{a \in \Ac} \langle a, \theta+\xi \rangle$ of a randomly perturbed objective. The small perturbation controls the mutual information between $\tilde{A}$ and $\theta$ without substantially degrading decision quality. 
It is easy to implement probability matching with respect to $\tilde{A}$ whenever linear optimization problems over $\tilde{A}$ are efficiently solvable. In particular, if $\mu_t = \E[\theta \mid\hist]$ and $\Sigma_t = \E[(\theta-\mu_t)(\theta - \mu_t)^\top \mid \hist]$ denote the posterior mean and covariance matrix, which are efficiently computable using Kalman filtering,  then by sampling $\hat{\theta}_t  \sim N(\mu_t, \Sigma_t)$ and $\hat{\xi}_t\sim  N(0, \Sigma_\xi)$ and setting $A_t = \argmax_{a\in \Ac} \langle a, \hat{\theta}_t + \hat{\xi}_t \rangle$ one has $ 
\Prob(A_t \in \cdot \mid \hist) = \Prob(\tilde{A} \in \cdot \mid \hist)$.

The result in the next theorem assumes the action set is contained within an ellipsoid  $\{ x\in \mathbb{R}^{p}  : x^\top Q^{-1} x \leq 1 \}$ and the resulting bound displays a logarithmic dependence on the eigenvalues of $Q$. Precisely, note that the trace of the matrix $Q$, or sum of its eigenvalues, provides one natural measure of the size of the ellipsoid. Our result also depends on the covariance matrix $\Sigma$ through the ${\rm Trace}(Q\Sigma)$. To understand this, consider applying similarity transforms to the parameter and action vectors so that $\theta' = \Sigma^{-1/2} \theta$ is isotropic and the set of action vectors is $\Ac'=\{\Sigma^{1/2}a : a \in \Ac\}$. This transformed action space is contained in the ellipsoid  $\{x : x^{\top} Q'^{-1} x \leq 1 \}$, where $Q' = \Sigma^{1/2} Q \Sigma$. Then ${\rm Trace}(Q')={\rm Trace}(Q\Sigma)$ provides a measure of the size of this ellipsoid. 

\begin{thm}
	\label{th:perturbed objective}
	Suppose $\Ac$ is a compact subset of the ellipsoid $\Ac \subset \{ x\in \mathbb{R}^{p}  : x^\top Q^{-1} x \leq 1 \}$ for some real symmetric matrix $Q$ and suppose $\theta\sim N(\mu,\Sigma)$ follows a $p$--dimensional multivariate Gaussian distribution. Set
	\[ 
	\tilde{A}= \argmax_{a \in \Ac} \langle a \,,\, \theta+\xi \rangle
	\]
	where $\xi$ is independent of $\theta$ and $\xi\sim N(0,\beta^2 \Sigma)$. For $\beta= D/\sqrt{{\rm Trace}(Q\Sigma)}$, 
	\[ 
	\E\left[\langle\theta \,,\, \tilde{A} \rangle\right] \geq \E \left[\langle \theta \,,\, A^* \rangle \right] -D
	\]
	and
	\[ 
	\Rc(D)\leq I(\tilde{A}; \theta) \leq \frac{p}{2} \log\left(1+ \frac{{\rm Trace}(Q\Sigma)}{D^2}\right).
	\]	
\end{thm}
\begin{proof} 
By Jensen's inequality
\[
\E[ \langle\tilde{A} \,,\, \theta+\xi\rangle] = \E \left[\max_{a \in \Ac} \langle a \,,\, \theta+\xi\rangle \right] \geq   \E \left[\max_{a \in \Ac} \langle a \,,\, \theta \rangle \right] = \E \left  \langle A^* \,,\, \theta\rangle \right].
\]
This implies
\[ 
\E \left[  \langle A^* \,,\, \theta\rangle \right] - \E \left[  \langle \tilde{A} \,,\, \theta\rangle \right] \leq \E[ \langle\tilde{A} \,,\, \xi\rangle] \leq \E \left[\max_{a \in \Ac} \langle a \,,\, \xi \rangle \right] \leq \E \left[  \max_{x \,: \|x\|_{Q^{-1}} \leq 1 } \langle a \,,\, \xi \rangle \right] = \E\left[\| \xi\|_{Q} \right]
\]
where the final equality uses the explicit formula for the maximum of a linear function over an ellipsoid. Now, 
\[ 
\E\left[\| \xi\|_{Q} \right] \leq \sqrt{\E[\xi^\top Q \xi] } = \sqrt{\E[{\rm Trace}(\xi^\top Q \xi)] }=\sqrt{{\rm Trace}( Q \E[\xi \xi^\top])} = \sqrt{\beta^2 {\rm Trace}( Q \Sigma)}= D. 
\]
Next we derive the bound on mutual information.  We have
\begin{eqnarray*} 
	I(\tilde{A}; \theta) \leq I(\theta+W; \theta) &=& H(\theta+W)- H(\theta+W | \theta)\\
	&=&   H(\theta+W)-H(W) \\
	&=& \frac{1}{2}\log\left( \frac{\det(\Sigma+\beta^2\Sigma)}{\det(\beta^2\Sigma)} \right)\\
	&=&\frac{1}{2}\log\left( \frac{\det((1+\beta^2)\Sigma)}{\det(\beta^2\Sigma)} \right)\\
	&=& \frac{p}{2} \log\left(1+\frac{1}{\beta^2} \right)\\
	&=&\frac{p}{2} \log\left(1+\frac{{\rm Trace}( Q \Sigma)}{D^2} \right)
\end{eqnarray*}
where $\det(\cdot)$ denotes the determinant of a matrix. Here the first inequality uses the data processing inequality, 
the third equality uses the explicit form the entropy of a multivariate Gaussian ($H(\theta)=\frac{1}{2}\log\left(2\pi e \det(\Sigma) \right)$) and the penultimate equality uses that $\det(c\Sigma)=c^p \det(\Sigma)$ for any scalar $c$.  
\end{proof}

\section{Application to the Infinite-Armed bandit.}\label{sec: satisficingTS}
This section considers a generalization of the deterministic infinite-armed bandit  problem in the introduction that allows for noisy observations and non-uniform priors. The action space is $\mathcal{A} = \{1,2,\ldots\}$. We assume $R_t \in [0,1]$ almost surely and $Y_t=R_t$, meaning the agent only observes rewards. The mean reward of action $a$ is $\mu(\theta, a) = \theta_a\in [0,1]$ where the prior distribution of $\theta_a$ is independent. Let $R^*$ denote the sepremal value in the support of $\theta_a$, so $\sup_{a \in \Ac} \theta_a =R^*$ almost surely.

\subsection{STS for the Infinite-Armed Bandit Problem.}\label{subsec: sts for infinite arms}
We consider the simple satisficing action defined in the introduction: $\tilde{A} = \min\{a \in \Ac: \theta_a \geq R^* - D\}$.  Rather than continue to explore until identifying the optimal action $A^*$, we will settle for the \emph{first}\footnote{ Of course, there is nothing crucial about this ordering on actions. We can equivalently construct a randomized order in which actions are sampled; for each realization of the random variable $\xi$, let $\pi_{\xi} : \N \to \N$ be a permutation, and take $\tilde{A} = \min \{\pi_{\xi}(a) : \theta_a \geq 1-D \}$.} action known to yield reward within $D$ of optimal.

We study satisficing Thompson sampling where actions are selected by probability matching with respect to $\tilde{A}$. Note that an algorithm for this problem must decide whether to sample a previously tested action -- and if so which one to sample -- or whether to try out an entirely new action. Let $\Ac_t = \{A_0,\ldots, A_{t-1}\}$ denote the set of previously sampled actions. STS may sample an untested action $A_t \notin \Ac$, and does so with probability
\[ 
\Prob(A_t \notin \Ac_{t} | \hist ) = \Prob(\tilde{A}\notin \Ac_{t} | \hist).
\]
equal to the posterior probability no satisfactory action has yet been sampled. The remainder of the action probabilities are allocated among previously tested actions, with
\[ 
\Prob(A_t = a | \hist) = \Prob(\tilde{A} =a | \hist) \qquad \forall a \in \Ac_t. 
\]

There is a simple algorithmic implementation of STS that mirrors computationally efficient implementations of Thompson sampling (TS). At time $t$, TS selects a random action $A_t$ via probability matching with respect to $A^*$. Algorithmically, this is usually accomplished by first sampling $\hat{\theta}_t \sim \Prob(\theta \in \cdot | \hist)$ and solving for $A_t \in \argmax_{a \in \mathcal{A}} \mu(a, \hat{\theta}_t)$. Similarly, we can implement STS by sampling and approximately optimizing a posterior sample. Over each $t$th period, STS selects an action $A_t$ as follows: \\
\begin{enumerate}
	\item For each $a\in \Ac_t$, sample $\hat{\theta}_a \sim \Prob(\theta_a \in \cdot| \hist)$
	\item Let $\hat{\tau} = \min\left\{\tau \in \{1,\ldots,t-1\} :   \mu(A_\tau, \hat{\theta}_t) \geq R^*-D \right\}$
	\item If $\hat{\tau}$ is not null set $A_t=A_{\hat{\tau}}$. Otherwise, play an untested action $A_t\notin \Ac_t$.  
\end{enumerate}
\vspace{\baselineskip}
Note that $D \geq 0$ is supplied to the algorithm as a tolerance parameter.  When $D = 0$, STS is equivalent to TS.  Otherwise,
STS attributes preference to selecting previously selected actions, which can yield substantial benefit in the face of time preference. 

This definition can be generalized to treat problems with a large, but finite, number of independent arms. Define the satisficing action $\tilde{A} = \min\{ a \in \Ac \mid \theta_a \geq \sup_{a' \in \Ac} \mu(\theta, a')  -D \}$. For the infinite armed bandit, $\sup_{a' \in \Ac} \mu(\theta, a')  =R^*$ with probability 1, but for problems with a finite number of arms $\max_{a' \in \Ac} \mu(\theta, a')$ is not known apriori. One can efficiently sample from this satisficing action by modifying step 2 above with the alternative definition $\hat{\tau} = \min\left\{\tau \in \{1,\ldots,t-1\} :   \mu(A_t, \hat{\theta}_t) \geq \sup_{a \in \Ac} \mu(a, \hat{\theta}_t) -D \right\}$.

\subsection{Computational Comparison of STS and TS.}
\label{se:computations}

We close with a simple computational illustration of the potential benefits afforded by STS. We consider two many-armed bandit problems, and demonstrate that per-period regret of STS diminishes much more rapidly than that of TS over early time periods. 

We consider problems with 250 actions, where the mean reward $\theta_a$ associated with each action $a \in \{1,\ldots,250\}$ is independently sampled uniformly from $[0,1]$.  We first consider the many-armed deterministic bandit, for which there is no observation noise.  Figure \ref{fig:computational-results}(a) presents per-period regret of TS and STS over 500 time periods, averaged over 5000 simulations, each with an independently sampled problem instance. STS is applied with tolerance parameter 0.05.  We next consider incorporating observation noise. In particular, instead of observing $\theta_a$, after selecting an action $a$, we observe a binary reward that is one with probability $\theta_a$. Figure \ref{fig:computational-results}(b) displays the results of this experiment.

\begin{figure}[h!]
	\centering
	\begin{subfigure}{.5\textwidth}
		\centering
		\includegraphics[height=2.1in]{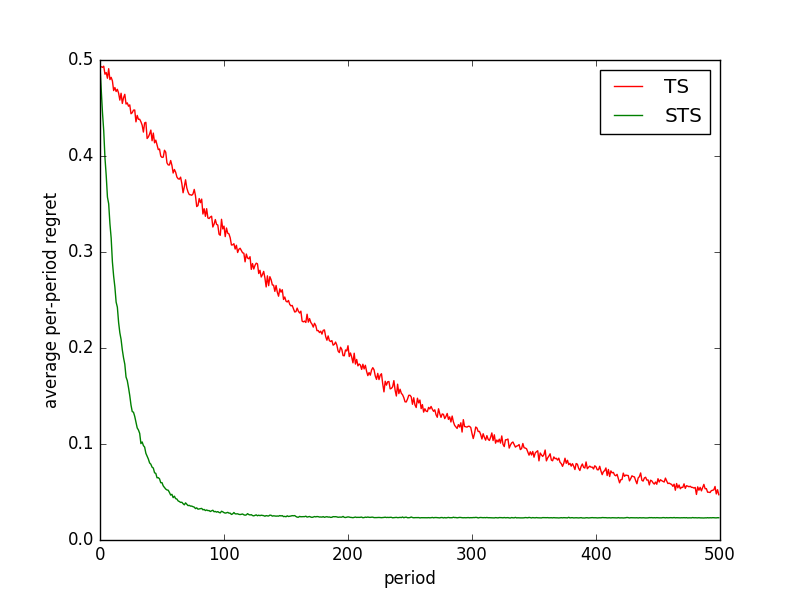}
		\label{fig:pic1}
	\end{subfigure}%
	\begin{subfigure}{.5\textwidth}
		\centering
		\includegraphics[height=2.1in]{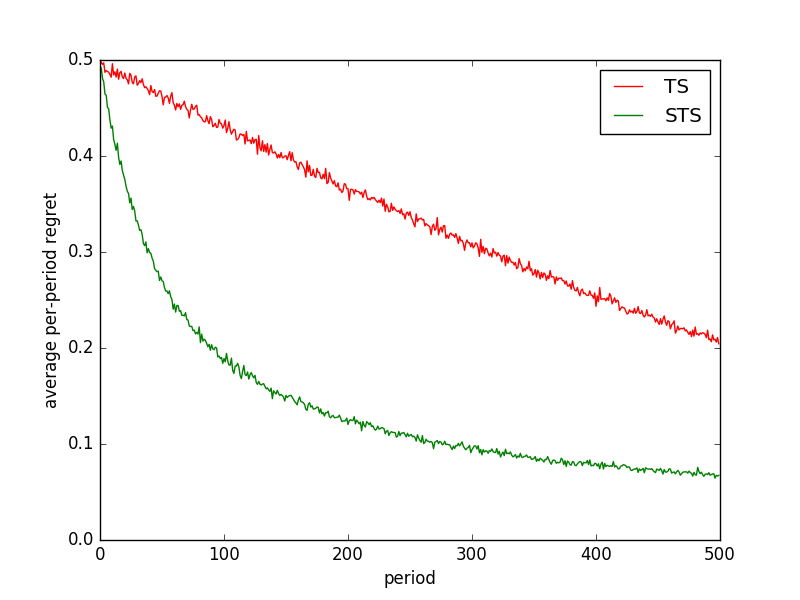}
		\label{fig:pic2}
	\end{subfigure}
	\caption{TS versus STS for the (a) many-armed deterministic bandit and (b) many-armed bandit with observation noise.}
	\label{fig:computational-results}
\end{figure}

\subsection{Information Ratio Analysis of the Infinite-Armed Bandit.}\label{subsec: infinite noisy}

The following theorem provides a discounted regret bound for STS in the infinite armed bandit.
The result follows from bounding the problems information ratio and the mutual information $I(\tilde{A}; \theta)$ and applying the general regret bound of Theorem \ref{th:discounted-regret}. This requires substantial additional analysis, the details of which are delayed until Subsection \ref{subsec: proof of of infinite armed bandit bound}.
\begin{thm}
	\label{thm:noisy regret bound} Consider the the infinite-armed bandit with noisy observations, and let $\tilde{A} = \min\{a \in \Ac: \theta_a \geq R^*-D\}$. Denote the STS policy with respect to $\tilde{A}$ by $\psi_D^{\rm STS}$. Then, 
	\[
	I( \tilde{A} ; \theta) \leq 1+ \log(1/\delta) \quad \mathrm{and} \quad \Gamma\left(\tilde{A}, \psi_D^{\rm STS} \right)  \leq 6+ 4/\delta  + (2/\delta)\log\left( \frac{1}{1-\alpha^2} \right)
	\]
	where $\delta=\Prob(\theta_a \geq R^*- D)$. Together with Theorem \ref{th:discounted-regret} this implies
	\begin{align}\nonumber
	\sregret(\alpha, \psi_D^{\rm STS}, D)
	&\leq \min\left\{ \sqrt{ \frac{\left(6+ 4/\delta  + (2/\delta)\log\left( \frac{1}{1-\alpha^2} \right) \right)(1+\log(1/\delta))   }{1-\alpha^2}} , \, \frac{1}{1-\alpha} \right\} \\ \label{eq: noisy regret bound simplified}
	&= \tilde{O}\left( \min\left\{ \sqrt{\frac{ 1/\delta   }{1-\alpha}} \, , \, \frac{1}{1-\alpha} \right\} \right).
	\end{align}
\end{thm}
The final expression \eqref{eq: noisy regret bound simplified} uses that 
$1-\alpha \leq 1-\alpha^2 \leq 2(1-\alpha)$.  The upper bound of $1/(1-\alpha)$ is naive and applies to all algorithms. The bound of order $\sqrt{\frac{ 1/\delta   }{1-\alpha}}$ requires an intelligent balance of exploration and exploitation; it is much stronger when the prior probability $\delta$ an arm is a satisficing arm is not too small.

\subsection{A Lower Bound for the Infinite-Armed Bandit}

This section establishes a lower bound on regret that matches the scaling of the upper bound in Theorem \ref{thm:noisy regret bound}. In this problem, $\epsilon$ is the fraction of arms that are exactly optimal. For our purposes, the interesting regime is where $\epsilon \ll 1-\alpha$ but $\delta \gg 1-\alpha$, in which case identifying an optimal arm is hopeless but it is worthwhile to search for a satisficing arm. In that regime, Theorem \ref{thm:noisy regret bound} shows a bound on satisficing regret of $\tilde{O}\left( \sqrt{\frac{1/\delta}{1-\alpha} } \right)$ and Theorem \ref{thm: lower bound} shows this is unimprovable in general. The specific threshold on $\epsilon$ in the theorem statement was chosen for analytical convenience and could likely be tightened.

\begin{thm}\label{thm: lower bound}
	Fix any $\alpha\in (0,1)$, $\delta \in (0, 1/2)$ and $D \in (0,1/4)$. Consider an instance of the infinite armed bandit problem in which, for all $a\in \Ac$,
	\begin{align}\label{eq: lower bound construction}
	\Prob\left(\theta_a = \frac{1}{2}-\Delta \right) = 1-\delta  && 
	\Prob\left(\theta_a = \frac{1}{2} \right)  = \delta - \epsilon  &&
	\Prob\left(\theta_a = \frac{1}{2}+D \right)  = \epsilon, 	
	\end{align}
	for $\Delta=\min\left\{\frac{1}{4}  \, ,  \frac{1}{4\sqrt{2} } \cdot \sqrt{\frac{1-\alpha }{\delta}}   \right\}$ and $\epsilon \leq \frac{1}{36}  \cdot \min\left\{ (1-\alpha)^2  \, , \,      \frac{(1-\alpha)^3}{2\delta}  \right\}$.
	Suppose $R_t \in \{0,1\}$ with $\Prob(R_t = 1 \mid \theta, A_t, \hist) = \theta_{A_t}$. Then,
	\[
	\inf_{\psi} \,\sregret(\alpha, \psi, D)  \geq    \frac{1}{32} \cdot \min\left\{  \frac{1}{1-\alpha} \, , \,    \sqrt{\frac{1/2\delta}{1-\alpha}}  \right\}
	\] 
	
	where the infimum is over all adaptive policies. 
\end{thm}

\subsection{Open Question: Gap Dependent Analysis of STS.}
It should be noted that Theorem \ref{thm: lower bound} is a worst-case construction. It shows that, for a given discount factor $\alpha$, satisficing level $D$ and fraction of satisficing arms $\delta$, there exists a hard instance of an infinite armed bandit problem in which the scaling of \eqref{eq: noisy regret bound simplified} is unavoidable. Stronger performance guarantees are likely possible for more benign problems, however. Here, we highlight one open problem in this direction. 

As shown in \cite{berry1997bandit} and \cite{wang2009algorithms}, when the agent begins with a uniform prior over each $\theta_a$, it is possible to attain undscounted regret that scales as $\E\left[ \sum_{t=1}^{T} (R^*-R_t) \right] = O\left( \sqrt{T} \right)$. While the worst-case construction in Theorem \ref{thm: lower bound} shows that Theorem \ref{thm:noisy regret bound} is tight without additional assumptions, it seems to yield an overly conservative regret bound of $O\left(T^{2/3} \right)$ for this specific problem. To understand this, note that if the agent begins with a uniform prior over each $\theta_a$, we find $\delta=D$. Choosing $D\approx (1-\alpha)^{2/3}$ to minimize the regret upper bound from combining Theorem \ref{thm:noisy regret bound} with Equation \eqref{eq: discounted-regret}, we find ${\rm Regret}(\alpha, \psi^{D}) \leq \tilde{O}\left( \frac{1}{(1-\alpha)^{2/3}}  \right),$ where $\psi^{D}$ is denotes satisficing Thompson sampling applied with parameter $D$. Since $\frac{1}{1-\alpha}$ is the effective time horizon in the problem, this roughly corresponds to a regret bound of $\tilde{O}(T^{2/3})$. 

Simulations suggest that such a bound is conservative, and STS actually attains the optimal $\Theta(\sqrt{T})$ regret scaling in this problem. To test this, we applied STS over a range of horizons $T\in \{ e^{5}, e^{5.5},\cdots, e^{10.5}\}$. For each choice of horizon, we used the satisficing parameter $D=3/\sqrt{T}$ and ran 500 independent trials. The numerical constant $3$ was selected in a somewhat ad-hoc manner and may be further optimized by tuning it in simulation. Figure \ref{fig:regret scaling} seems to suggest 
$\E\left[ \sum_{t=1}^{T} (R^*-R_t) \right] = \Theta\left( \sqrt{T} \right)$, since the logarithm of regret scales linearly as $\log(T)/2$. An analogous experiment in the discounted setting suggests ${\rm Regret(\alpha, \psi_D)} = \Theta\left( \frac{1}{\sqrt{1-\alpha}}  \right)$ as $\alpha \to 1$. 

To understand the two different scalings of regret, it may be helpful to draw an analogy to the standard $k$ armed bandit problem. When there is a fixed gap of $\Delta>0$ between the best and second best arm, it is possible to provide regret bounds of $O( \log(T) /\Delta )$ which scale very slowly with the time horizon but degrade as $\Delta$ shrinks. In worst-case instances, $\Delta \approx 1/\sqrt{T}$ is small relative to the horizon and regret bounds of $O(\sqrt{T})$, which are completely independent of the gap, are the best possible. Our analysis in Theorem \ref{th:discounted-regret} was effectively gap-independent, as it did not make assumptions about the separation between satisficing actions a non-satisficing actions. By imposing the additional assumption that each $\theta_a$ is drawn from a uniform prior, we ensure that most arms are easily distinguished from the satisficing arms--especially as the satisficing threshold $D$ tends to zero--- which is what makes the $O(\sqrt{T})$ bound possible. One may be able to leverage finite time gap-dependent analyses of Thompson sampling \cite{agrawal2013further} to show an $O(\sqrt{T})$ regret bound for STS under a restricted class of priors, but we leave this as an open question that is beyond the scope of this work. 

\begin{figure}[h!]
	\centering
	\begin{subfigure}{.5\textwidth}
		\centering
		\includegraphics[height=2.1in]{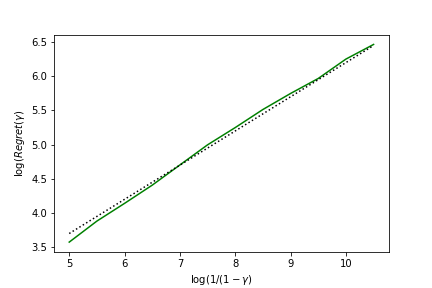}
	\end{subfigure}%
	\begin{subfigure}{.5\textwidth}
		\centering
		\includegraphics[height=2.1in]{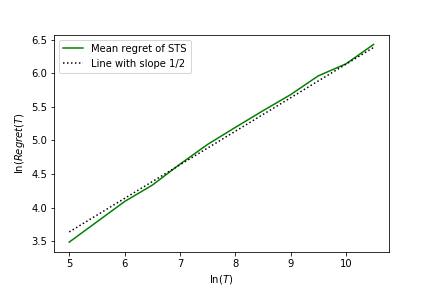}
	\end{subfigure}
	\caption{Scaling of regret in an infinite armed bandit with uniform prior for (a) a discounted infinite horizon and (b) an undiscounted finite horizon problem.}
	\label{fig:regret scaling}
\end{figure}

\section{Application of STS to the Hierarchical Infinite-Armed bandit.}\label{sec: Hierarchical}

This section offers a preliminary study of satisficing in the hierarchical infinite armed bandit as described in Example 2 of Section \ref{sec:intro}. Algorithms that succeed on this example must simultaneously leverage prior beliefs, generalize across arms, and satisfice.  

Example \ref{ex: hiearchical satisficing action} of Section \ref{sec: satisficing} describes a satisficing action $\tilde{A}$ for this problem. Precisely, according to Equation \eqref{eq: hierarchical satisficing}, $\tilde{A} = \min\{a \in \Ac(\theta_0) \mid \theta_a \geq 1-D \}$ where $\Ac(\theta_0) =  \argmax_{a' \in \Ac} \langle \phi(a') \, , \, \theta_0 \rangle  \}$. Satisficing Thompson sampling performs probability matching with respect to $\tilde{A}$, setting $\Prob(A_t = a \mid \hist )= \Prob(\tilde{A}=a \mid \hist)$ for each $a\in \Ac$.

Let us describe how to implement probability matching in a manner that mirrors the description of STS for the infinite armed bandit. Let $\Ac_t = \{ A_0, \ldots  A_{t-1} \}$ denote the set of previously sampled actions and $\Ac_t(\theta_0) = \Ac_{t} \cap \Ac(\theta_0)$ denote previously sampled actions whose feature vectors are optimal under parameter $\theta_0\in \mathbb{R}^d$. Let $\theta_{1\Ac_t(\theta_0)} =\{\theta_{1a} : a \in \Ac_{t}(\theta_0)  \}$ denote the corresponding components of $\theta_{1}$. Over each $t$th period, STS selects an action $A_t$ as follows: 
\begin{enumerate}
	\item Sample $\hat{\theta}_0 \sim \Prob(\theta_0 \in \cdot \mid \hist )$
	\item Sample $\hat{\theta}_{1\Ac_t(\hat{\theta}_0)} \sim \Prob\left( \theta_{1\Ac_t(\theta_0)} \in \cdot \mid \theta_0=\hat{\theta}_0, \hist  \right)$  
	\item Let $\hat{\tau} = \min\left\{\tau \in \{1,\ldots,t-1\} :   \hat{\theta}_{1 A_{\tau}} \geq 1-D \right\}$
	\item If $\hat{\tau}$ is not null set $A_t=A_{\hat{\tau}}$. Otherwise, play an untested action $A_t \in \Ac(\hat{\theta}_0)\setminus \Ac_t$.  
\end{enumerate}
Steps 2-4 correspond to satisficing Thompson sampling for the infinite armed bandit as described in Section \ref{sec: satisficingTS}, but applied conditioned on $\theta_0=\hat{\theta}_0$. Using the ideas described in Subsection \ref{subsec: sts for infinite arms}, a simple modification of this applies to problems with a large but finite number of arms and a prior that is not necessarily uniform. We run STS with the satisficing action 
\begin{equation}\label{eq: STS finite hierarchical}
\tilde{A}  = \argmin_{ a' \in \mathcal{A}(\theta_0)} \left\{ a' \mid \theta_{1a'} \geq \max_{a \in \Ac(\theta_0)} \theta_{1a}  -D  \right\}
\end{equation} 
As in Subsection \ref{subsec: sts for infinite arms}, we can sample from this action by applying the steps above with a modified definition of $\hat{\tau}$. 

We run a simple numerical experiment to demonstrate the importance of satisficing and generalization to this problem. To simplify the implementation, we focus on the case where there is a normally distributed prior over $\theta_0$ and each $\theta_{1a}$,  allowing the posterior distribution to be expressed in closed form. The experiment treats a problem with a large but finite number of arms, highlighting that satisficing is just as important in such settings. 

Results are displayed in Figure \ref{fig:hierarchical}. The simulation focuses on performance over the first $T=50$ periods. The dimension of the linear model is $d=2$. There is no observation noise, so the reward at time $t$ is $\langle \phi(A_t), \theta_0 \rangle + \theta_{A_t}$.  There are  $k =400$ of actions each of which has a feature vector $\phi(a) \in \{0,1\}^2$ that is drawn randomly. Parameters are drawn randomly under a prior with $\theta_0 \sim N(0, d^{-1/2} I_d)$ and $\theta_{1,a} \sim N(0, 1)$ independently across arms $a$ and from $\theta_0$. We run STS with the satisficing action in \eqref{eq: STS finite hierarchical}. 

We compare the performance of STS against Thompson sampling (TS) and an incorrect version of STS that does not model the linear component of the model, instead treating the problem as an infinite armed bandit independent arms and normally distributed prior. We choose $D=1$ for both variants of STS. Figure \ref{fig:hierarchical} shows the average reward earned by each algorithm in each of the first 50 periods. Results are averaged across 2000 trials\footnote{Standard errors are smaller than .01 for each algorithm and time period, so confidence intervals are omitted.}. We see that STS earns much higher rewards in early periods, earning an average reward across the 50 time periods of 1.21, compared with .79 for incorrect STS and .67 for standard Thompson sampling. In this case Thompson sampling generalizes across arms but does not satisfice, preventing it from settling on arms with large arm-specific effect. The incorrect variant of STS satisfices, but does not generalize across arms. STS can simultaneously generalize and satisfice, leading to large improvement in performance. 

\begin{figure}
	\centering
	\includegraphics[height=2.1in]{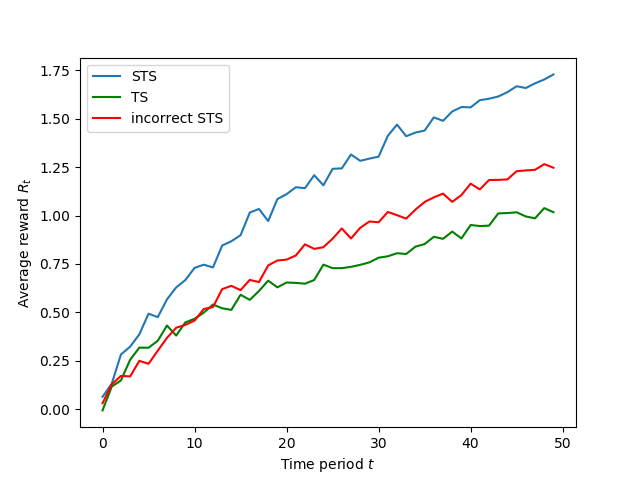}
	\caption{Average reward earned in early periods of the hierarchical bandit problem.}
	\label{fig:hierarchical}
\end{figure}


\section{Information Theoretic Proofs of the Infinite Armed Bandit Results}
\subsection{Proof of the upper bound: Theorem \ref{thm:noisy regret bound}.}
\label{subsec: proof of of infinite armed bandit bound} 
Our proof will use the following fact, which is a consequence of Pinsker's inequality and is stated as Fact 9 in \cite{russo2016info}. \\
\begin{fact}\label{fact: DME to DKL} For any distributions $P$ and $Q$ such that that $P$ is absolutely continuous with respect to $Q$,  any random variable $X: \Omega \rightarrow \mathcal{X}$ and any $g:\mathcal{X}\rightarrow \mathbb{R}$ such that $\sup g  - \inf g \leq 1$, 
	$$\E_{P} \left[ g(X) \right] - \E_{Q} \left[ g(X) \right]  \leq \sqrt{\frac{1}{2} D \left( P || Q \right) },$$
	where $\E_{P}$ and $\E_{Q}$ denote the expectation operators under $P$ and $Q$. 
\end{fact}
\vspace{\baselineskip}

We begin by showing the mutual information bound stated as part of Theorem \ref{thm:noisy regret bound}. 
\begin{lemma}[Mutual Information Bound]
	\label{lem: infinite armed bandit mutual info bound}
	Let $\delta=\Prob(\theta_a \geq 1-D)$. Then
	\[
	I(\tilde{A}; \theta) \leq 1 + \log\left( 1/ \delta \right).
	\] 
\end{lemma}
\begin{proof} 
Since $\tilde{A} = \min\{a \in \Ac: \theta_a \geq 1-D\}$ is a deterministic function of $\theta$, we have 
$I(\tilde{A}; \theta)=H(\tilde{A}) = H(N)$ where $N\sim {\rm Geom}(\delta)$ is a geometric random variable. This implies
\begin{eqnarray*}
	I(\tilde{A} ; \theta)
	&=& H\left( N \right) \\
	&=& -\sum_{k=1}^{\infty} \delta (1-\delta)^{k-1} \log(\delta (1-\delta)^{k-1}) \\
	&=& -\sum_{k=1}^{\infty}\delta (1-\delta)^{k-1}\log(\delta) - \sum_{k=1}^{\infty}\delta(1-\delta)^{k-1}\log((1-\delta)^{k-1})\\
	&=& \sum_{k=1}^{\infty}\Prob(N=k)\log(1/\delta)  - \log(1-\delta) \sum_{k=1}^{\infty} \delta (1-\delta)^{k-1}(k-1) \\
	&=& \log(1/\delta) + \log\left( \frac{1}{1-\delta} \right)(\E[N]-1)\\
	&=& \log(1 /\delta) + \log\left(1 + \frac{\delta}{1-\delta}  \right)\left( \frac{1-\delta}{\delta} \right) \\
	&\leq & \log(1/\delta)+ \left( \frac{\delta}{1-\delta} \right)\left( \frac{1-\delta}{\delta} \right)\\
	&=& 1+\log(1/\delta). 
\end{eqnarray*}	 
\end{proof}

Throughout the remainder of the proof we use the notation 
\[ 
\Gamma_{t} = \frac{(\E_{t}[\theta_{\tilde{A}} - \theta_{A_t}])^2}{I_{t}(\tilde{A} ; (A_t ,Y_{t}) )}.
\]
This represents the expected one-step information ratio in period $t$ under the posterior measure $\Prob(\cdot | \hist)$. The next lemma shows that the cumulative information ratio can be bounded by the expected discounted average of these one-step information ratios.  
\begin{lemma}[Relating the information ratio to the one-step-information-ratio]
	\label{lem: one step to total information ratio}
	\[
	\Gamma\left(\tilde{A}, \psi_D^{\rm STS}	\right) \leq (1-\alpha^2)\sum_{t=0}^{\infty} \alpha^{2t} \E[ \Gamma_{t}]
	\]
\end{lemma}
\begin{proof} 
We have
\begin{eqnarray*}
	\E[\tilde{R} - R_t] = \E\left[\theta_{\tilde{A}} - \theta_{A_t}\right] &=&   \E\left[  \E_{t}\left[\theta_{\tilde{A}} - \theta_{A_t}\right] \right]  \\
	&=& \E\left[ \sqrt{\Gamma_{t}} \sqrt{I_{t}\left(\tilde{A};  (A_t, Y_{t})\right)} \right]  \\
	&\leq& \sqrt{ \E[\Gamma_t] \, \E\left[ I_{t}\left(\tilde{A};  (A_t, Y_{t})\right) \right]} \\
	&=& \sqrt{ \E[\Gamma_t] \,  I\left(\tilde{A};  (A_t, Y_{t}) \mid \hist \right). }
\end{eqnarray*}
Then, by the definition of the information ratio
\[
\Gamma\left(\tilde{A}, \psi \right) = (1-\alpha^2) \sum_{t=0}^\infty \alpha^{2t} \left(\frac{(\E[\tilde{R} - R_{t}])^2}{I(\tilde{A};  (A_t, Y_{t}) | \hist )}\right)  \leq (1-\alpha^2)\sum_{t=0}^{\infty} \alpha^{2t} \E[ \Gamma_{t}]. 
\]
\end{proof}  
The next lemma provides a bound on the one-step information ratio. 
\begin{lemma}
	\[
	\Gamma_{t} \leq 2 |\Ac_{t}| +2/\delta
	\]
	where $\Ac_{t} = \cup_{s=1}^{t-1}\{ A_s \}$ is the set of actions that were sampled before period $t$, and  $\delta \equiv \Prob(\theta_{i} \geq R^*-D)$ is the prior probability an arm is $D$--optimal.
\end{lemma}
\begin{proof}
Define
\[
L \equiv \E[\theta_{i} | \theta_{i} \geq R^*-D] -\E[\theta_{i}]
\]
and
\[
\delta \equiv \Prob(\theta_{i} \geq R^*-D).
\]
Here $\delta$ is the probability an unsampled arm is $D$--optimal, and $L$ is the difference between the expected reward of a $D$--optimal arm and that of an arm sampled uniformly at random. In the case where $\theta_i \sim {\rm Unif}(0,1)$, $\delta=D$ and $L= (1-D)/2$.

We can write expected regret as
\begin{eqnarray*}
	\E_{t}[\theta_{\tilde{A}} - \theta_{A_t}] &=& \sum_{a \in \Ac}\Prob_{t}(\tilde{A}=a)\E_{t}[\theta_{a}|\tilde{A}=a] -  \sum_{a \in \Ac}\Prob_{t}(A_t=a)\E_{t}[\theta_{a}] \\
	&=&\sum_{a \in \Ac_{t}}\Prob_{t}(\tilde{A}=a)\left(\E_{t}[\theta_{a}|\tilde{A}=a]-\E_{t}[\theta_{a}]\right)\\
	&&+ \sum_{a \notin \Ac_{t}}\Prob_{t}(\tilde{A}=a)\E_{t}[\theta_a | \tilde{A}=a]-  \sum_{a \notin \Ac_t}\Prob_{t}(A_t=a)\E_{t}[\theta_{a}]\\
	&=& \sum_{a \in \Ac_{t}}\Prob_{t}(\tilde{A}=a)\left(\E_{t}[\theta_{a}|\tilde{A}=a]-\E_{t}[\theta_{a}]\right)
	+ \Prob_{t}(\tilde{A} \notin \Ac_{t})(\E[\theta_a | \theta_a \geq 1-D]- \E[\theta_a]) \\
	&=& \underbrace{\sum_{a \in \Ac_{t}}\Prob_{t}(\tilde{A}=a) \left(\E_{t}[\theta_{a}|\tilde{A}=a] - \E_{t}[\theta_{a}]\right)}_{\Delta_{t,1}} + \underbrace{\Prob_{t}(\tilde{A} \notin \Ac_{t})L}_{\Delta_{t,2}}.
\end{eqnarray*}
This decomposes regret into the sum of two terms: one which captures the regret due to suboptimal selection within the set of previously sampled actions $\Ac_{t}$, and one due to the remaining possibility that none of the sampled actions are $D$--optimal. The proof develops a similar decomposition for mutual information, and then lower bounds both terms.

Let $Y_{t,a}=g(a, \theta, W_t)$ denote the outcome that would have been realized from a choice of action $a$ at time $t$. We can express mutual information as follows:
\begin{eqnarray*}
	I_t( \tilde{A}; (A_t, Y_{t}) ) &=&\sum_{a\in \Ac} \Prob_{t}(A_t=a) I_t( \tilde{A}; Y_{t} | A_t =a) \\
	&=&\sum_{a\in \Ac_t} \Prob_{t}(\tilde{A}=a) I_{t}(\tilde{A} ; Y_{t,a}) + \sum_{a\notin \Ac_t} \Prob_{t}(\tilde{A}=a) I_{t}(\tilde{A} ; Y_{t,a} | A_t =a)
\end{eqnarray*}
Let us focus on the second sum, which captures the information acquired due to sampling previously untested actions $a\notin \Ac_{t}$. Such an action provides information about $\tilde{A}$, since if $\tilde{A}\notin \Ac_t$ and   $\theta_a\geq 1-D$, then $a$ is the first sampled action to be sufficiently close to optimal and $\tilde{A}=a$. Using the shorthand $P_{t}(X) = \Prob_{t}(X \in \cdot)$  to denote the posterior distribution of a random variable $X$, we have that for an untested action $a\notin \Ac_{t}$ 
\begin{eqnarray*}
	I_{t}(\tilde{A} ; Y_{t,a} | A_t =a)& =&  \sum_{\tilde{a} \in \Ac} \Prob_{t}\left( \tilde{A}=\tilde{a}  \mid A_t =a \right) D\left(P_{t}(Y_{t, a}  \mid  \tilde{A} = \tilde{a}) \,||\, P_{t}(Y_{t,a})   \right) \\ 
	&\geq& \Prob_{t}\left( \tilde{A}=a  \mid A_t =a \right) D\left(P_{t}(Y_{t, a}  \mid  \tilde{A} = a) \,||\, P_{t}(Y_{t,a})   \right) \\
	&=& \Prob_{t}\left( \tilde{A}=a  \mid A_t =a \right) D\left(P_{t}(Y_{t, a}  \mid  \theta_{a}\geq 1-D  ) \,||\, P_{t}(Y_{t,a})   \right) \\
	&\geq &  2 \Prob_{t}\left( \tilde{A}=a  \mid A_t =a \right) \left( \E_{t}\left[\theta_{a} \mid \theta_a \geq 1-D \right]- \E_{t}\left[\theta_{a}\right]\right)^2\\
	&=& 2 \Prob_{t}\left( \tilde{A}=a  \mid A_t =a \right) L^2 \\
	&=& 2 \Prob_{t}\left( \tilde{A} \notin \Ac_t \right) \Prob_{t}\left( \tilde{A}=a  \mid A_t =a, \tilde{A}\notin \Ac_t \right)L^2 \\
	&=& 2 \Prob_{t}\left( \tilde{A} \notin \Ac_t \right) \delta L^2.
\end{eqnarray*}
The second inequality uses Fact \ref{fact: DME to DKL}. This implies
\[
\sum_{a\notin \Ac_t} \Prob_{t}(\tilde{A}=a) I_{t}(\tilde{A} ; Y_{t,a} | A_t =a) \geq 2 \Prob_{t}\left( \tilde{A} \notin \Ac_t \right)^2 \delta L^2.
\]

Next, following the proof of Proposition 3 of \citet{russo2016info} shows
\begin{eqnarray*}
	\sum_{a\in \Ac_{t}} \Prob_{t}(\tilde{A}=a)I_{t}(\tilde{A} ; Y_{t,a}) &=& \sum_{a\in \Ac_{t}} \Prob_{t}(\tilde{A}=a) \sum_{\tilde{a} \in \Ac} D\left(P_{t}(Y_{t,a} \mid \tilde{A}=\tilde{a}) \,||\, P_{t}(Y_{t,a}) \right)  \\
	&\geq & \sum_{a\in \Ac_{t}} \Prob_{t}(\tilde{A}=a)^2 D\left(P_{t}(Y_{t,a}\mid \tilde{A}=a) \,||\, P_{t}(Y_{t,a}) \right) \\
	&\geq & 2\sum_{a\in \Ac_{t}} \Prob_{t}(\tilde{A}=a)^2 \left(\E_{t}[\theta_{a} | \tilde{A}=a ] - \E_{t}[\theta_{a}]   \right)^2  \\
	& \geq  & \frac{2}{|\Ac_{t}|} \left( \sum_{a \in \Ac_{t}} \Prob_{t}( \tilde{A} = a)\left(\E_{t}[\theta_{a} | \tilde{A}=a ] - \E_{t}[\theta_{a}]   \right) \right)^2
\end{eqnarray*}
where the second inequality uses Fact \ref{fact: DME to DKL}. Therefore
\[
I_t( \tilde{A}; (A_t, Y_{t}) ) \geq \underbrace{\frac{2}{|\Ac_{t}| } \left( \sum_{a \in \Ac_{t}} \Prob_{t}( \tilde{A} = a)\left(\E_{t}[\theta_{a} | \tilde{A}=a ] - \E_{t}[\theta_{a}]   \right) \right)^2}_{G_{t,1}} +\underbrace{ 2\Prob_{t}(\tilde{A} \notin \Ac_{t})^2 \delta L^2}_{G_{t,2}},
\]
is lower bounded by the sum of two terms: one which captures the information gain due to refining knowledge about previously sampled actions, and one that captures the expected information gathered about previously unexplored actions.

To bound the information ratio we'll separately consider two cases. If $\Delta_{t,1} \geq \Delta_{t,2}$,
then
\[
\frac{(\E_{t}[\theta_{\tilde{A}} - \theta_{A_t}])^2}{I_t( \tilde{A}; (A_t, Y_{t}) )} \leq \frac{(2\Delta_{t,1})^2}{G_{t,1}+G_{t,2}} \leq \frac{4(\Delta_{t,1})^2}{G_{t,1}} = 2 |\Ac_{t}|.
\]
If instead $\Delta_{t,1} < \Delta_{t,2}$, then
\[
\frac{(\E_{t}[\theta_{\tilde{A}} - \theta_{A_t}])^2}{I_t( \tilde{A}; (A_t, Y_{t}) )} \leq \frac{(2\Delta_{t,2})^2}{G_{t,1}+G_{t,2}} \leq \frac{4(\Delta_{t,2})^2}{G_{t,2}} = \frac{2}{\delta}.
\]
This shows
\[
\frac{(\E_{t}[\theta_{\tilde{A}} - \theta_{A_t}])^2}{I_t( \tilde{A}; (A_t, Y_{t}) )} \leq 2 |\Ac_{t}| +2/\delta. 
\]
\end{proof} 
Combining this result with Lemma \ref{lem: one step to total information ratio}
gives the bound
\begin{eqnarray*}
	\Gamma\left(\tilde{A}, \psi^{\rm STS}_{D} \right) &\leq& (1-\alpha^2) \sum_{t=0}^\infty \alpha^{2t} \E\left[\Gamma_{t}\right] \\
	&\leq & 2/\delta + 2(1-\alpha^2) \sum_{t=0}^\infty \alpha^{2t} \E[ |\Ac _{t}|].
\end{eqnarray*}
To use this result, we begin by bounding $\E[|\Ac_{t}|].$
\begin{lemma}\label{lem: bound on number of sampled actions}
	$|\Ac_{0}|=0$ and for each  $T \in \{1,2,\ldots\},$  $\E[|\Ac_{T}|] \leq 2+ \log(T)/\delta$.
\end{lemma}
\begin{proof} 
Let $\tau_{k} = \min\{ t \leq T | |\Ac_{t}| \geq k   \}$ denote the first period before $T$ in which $k$ actions have been sampled. Then
\begin{eqnarray*}
	\E[|\Ac_{T}|]  &=& \E[|\Ac_{\tau_{k}}|]+ \E[|\Ac_{T}|- |\Ac_{\tau_{k}}|]\\
	&\leq & \E[|\Ac_{\tau_{k}}|]+ \E[|\Ac_{\tau_{k}+T}|- |\Ac_{\tau_{k}}|]\\
	&\leq & k + \E \sum_{t=\tau_{k}}^{\tau_{k}+T-1}  \mathbf{1}(A_t \notin \Ac_{t}  ) \\
	&=& k + \E \sum_{s=0}^{T-1} \Prob(A_{\tau_{k}+s} \notin \Ac_{\tau_{k}+s} | H_{\tau_k +s} ) \\
	&=&  k + \E \sum_{s=0}^{T-1} \Prob(\tilde{A} \notin \Ac_{\tau_{k}+s} | H_{\tau_k +s} ) \\
	&= & k + \sum_{s=0}^{T-1} \Prob(\tilde{A} \notin \Ac_{\tau_{k}+s}) \\
	& \leq & k + T \Prob(\tilde{A} \notin \Ac_{\tau_{k}}) \\
	&= & k + T \Prob( {\rm Geom}(\delta) > k)  \\
	&=& k + T( 1- \delta)^{k} \\
	&\leq &  k +Te^{-\delta k}.
\end{eqnarray*}
Choosing $k = \lceil \log(T) / \delta \rceil \leq 1 +\log(T) / \delta,$ implies
\[
\E[ |\Ac_{T}|] \leq  2 + \log(T) / \delta. 
\] 
\end{proof}  

The next technical lemma shows $\sum_{t=1}^{\infty} \gamma^{-t} \log(t)=O((1/\gamma)\log(1/\gamma)).$ The proof is given in Appendix \ref{se: geometric average of logs}.

\begin{lemma}\label{lem: geometric average of logs}
	For any $\gamma \in (0,1)$,
	\[
	\sum_{t=1}^{\infty} \gamma^{-t} \log(t)\leq \frac{1}{1-\gamma} \left[ 1+ \log\left(\frac{1}{1-\gamma}\right)\right].
	\]
\end{lemma}

Finally we can conclude with the proof of Theorem \ref{thm:noisy regret bound}. As shown before, 
\begin{eqnarray*}
	\Gamma\left(\tilde{A}, \psi^{\rm STS}_{D} \right) &\leq& (1-\alpha^2) \sum_{t=0}^\infty \alpha^{2t} \E\left[\Gamma_{t}\right] \\
	&\leq & 2/\delta + 2(1-\alpha^2) \sum_{t=0}^\infty \alpha^{2t} \E[ |\Ac _{t}|].
\end{eqnarray*}
By Lemma \ref{lem: bound on number of sampled actions} and Lemma \ref{lem: geometric average of logs} we find
\begin{eqnarray*}
	(1-\alpha^2) \sum_{t=0}^\infty \alpha^{2t} \E[ |\Ac _{t}|] &\leq & (1-\alpha^2) \sum_{t=1}^\infty \alpha^{2t} \left(2+ \log(t)/\delta\right) \\
	& \leq & 3 + (1/\delta)(1-\alpha^2) \sum_{t=1}^\infty \alpha^{2t} \log(t) \\
	& \leq  & 3 + (1/\delta)\left[ 1+ \log\left(\frac{1}{1-\alpha^2}\right)\right].
\end{eqnarray*}
This implies
\[
\Gamma\left(\tilde{A},\psi^{\rm STS}_{D} \right) \leq 6 + 4/\delta + (2/\delta)\log\left(\frac{1}{1-\alpha^2} \right) = O\left( (1/\delta) \log\left(\frac{1}{1-\alpha^2}\right)  \right)
\]
and concludes the proof of Theorem \ref{thm:noisy regret bound}.

\subsection{Proof of the Lower Bound: Theorem \ref{thm: lower bound}}
The lower bound analysis leverages many of the standard information theoretic techniques for estalbishing minimax lower bounds \cite[See e.g.][]{tsybakov2008introduction}. We first give a reduction to a hypothesis testing problem in which the goal is simply to identify a satisficing action. We show identifying a satisficing action is hard by upper bounding  a certain KL divegence, through the data-processing inequality and the chain rule. These techniques are based on a classical change of measure argument by \cite{lai1985asymptotically} as well as other bandit lower bounds \cite{kaufmann2016complexity, bubeck2012regret}. Our proof resembles most closely a proof in \citet{bubeck2012regret}  that the minimax regret for $k$-armed un-discounted stochastic bandits is lower bounded by $\frac{1}{20}\sqrt{kT}$, where $T$ is the number of time periods. There is some novelty to our lower bound analysis, however. The most significant change is that our problem involves an infinite number of arms and an independent prior, so there are many satisficing arms and we need to argue it is difficult to consistently play any arm from that set. The construction in \cite{bubeck2012regret} involves a problem instance with a dependent prior, under which it is difficult to identify the single arm that differs from the other $k-1$ arms. We also show how to carry out lower bound analyses of discounted problems by analyzing the distribution of $A_{\tau}$, the action chosen at some randomly selected time $\tau \sim {\rm Geom(1-\alpha)}$. 

\begin{proof} 

{\bf Step 1: Expressing discounted sums in terms of random times.}\\
Let $\tau \sim {\rm Geom}(1-\alpha)$ be distributed independently of all other random variables. In several places, we use that, by the explicit form of the PMF $\Prob(\tau = t) = \alpha^t (1-\alpha)$, 
\[
\mathbb{E} \sum_{t=0}^{\infty} \alpha^t f(A_t)  = (1-\alpha)^{-1} \E[f(A_\tau)]   
\]
for an arbitrary function $f: \mathbb{N} \to \mathbb{R}$. The allows us to compactly compress functions of the entire sequence of interactions into expectations with respect to $A_{\tau}$ alone.

{\bf Step 2: Regret lower bound in terms of the probability of satisficing.}\\
In this problem, the optimal expected reward is $R^* = \frac{1}{2}+D$. We put $S_{\theta}=\{ i \in \mathbb{N} \, : \,  \theta _i \geq \frac{1}{2} \}$ to be the set of satisficing actions and ${\rm OPT}_{\theta} =\{ i \in \mathbb{N} \, : \,  \theta _i= R^* \}$ to be the set of optimal actions. These sets are random due to their dependence on $\theta$. We have 

\begin{align*}
\sregret(\alpha, \psi, D) = \E\left[ \sum_{t=0}^{\infty} \alpha^{t}(R^* - R_t - D)\right] &=  \E\left[ \sum_{t=0}^{\infty} \alpha^{t}(\E\left[ R^* - R_t - D  \mid \hist \right] \right] \\
&=\E\left[ \sum_{t=0}^{\infty} \alpha^{t}(\E\left[ R^* - \theta_{A_t} - D  \mid \hist \right] \right] \\
&= \left(1-\alpha\right)^{-1} \E\left[ R^* - \theta_{A_\tau}  - D\right]\\
& = \left(1-\alpha\right)^{-1} \left[ \Delta\Prob\left(\theta_{{A}_{\tau}}=\frac{1}{2}-D\right) - D\Prob\left(\theta_{{A}_{\tau}}=\frac{1}{2}+D\right) \right]\\
&= \left(1-\alpha\right)^{-1} \left[ \Delta \Prob(A_\tau \in S_{\theta})   - D \Prob(\theta_{A_\tau} \in {\rm OPT}_\theta) \right] 
\end{align*}

We now upper bound $\Prob(A_\tau \in {\rm OPT}_\theta)$ in terms of $\epsilon$. Recall the definition $\Ac_t := \{A_0,\ldots, A_{t-1}\}$. Proceeding recursively, we have $\Prob(\theta_{A_1} < R^*) = 1-\epsilon$, since $A_1$ is chosen independently of $\theta$. Next, we have 
\begin{align*}
\Prob(\theta_{A_1} < R^*  \wedge \theta_{A_2} < R^* )&= (1-\epsilon)\Prob( \theta_{A_2} < R^*  \mid  \theta_{A_1} < R^*)\geq (1-\epsilon)^2.
\end{align*}
To understand the final inequality, note that $\Prob( \theta_{A_2} < R^*  \mid  \theta_{A_1} < R^*, A_1, A_2)$ is equal to one if $A_1=A_2$ and is equal to $1-\epsilon$ otherwise. Hence  $\Prob( \theta_{A_2} < R^*  \mid  \theta_{A_1} < R^*, A_1, A_2) \geq 1-\epsilon$ almost surely. Repeating this process inductively gives 
\[ 
\Prob(\theta_{A_1} < R^*  \wedge\cdots \wedge \theta_{A_t} < R^*) \geq (1-\epsilon)^t.
\]
This gives the bound
\begin{align*}
\Prob( A_\tau \notin {\rm OPT}_\theta) &\geq \Prob(\theta_{A_\tau} < R^*  \wedge\cdots \wedge \theta_{A_\tau} < R^*) \\
&= \E\left[ \Prob(\theta_{A_\tau} < R^*  \wedge\cdots \wedge \theta_{A_\tau} < R^* \mid \tau) \right]  \\
&\geq \E\left[ (1-\epsilon)^\tau \right] \\
&=\sum_{t=0}^{\infty} \alpha^{t} (1-\alpha) (1-2\epsilon)^t \\
&=\frac{1-\alpha}{1-\alpha(1-\epsilon)}. 
\end{align*} 
We find 
\[
\Prob( A_\tau \in {\rm OPT}_\theta) \leq 1-\frac{1-\alpha}{1-\alpha(1-\epsilon)} = \frac{\alpha \epsilon}{1-\alpha(1-\epsilon)}  \leq  \frac{ \epsilon}{ 1-\alpha}.
\]
We've reached our desired result, which lower bounds satisficing regret in terms of the probability of playing a satisficing arm at the random time $\tau$:
\begin{equation}\label{eq: lower bound by prob of satisficing}
\sregret(\alpha, \psi, D) \geq \Delta \cdot \left( \frac{ 1-\Prob( A_\tau \in S_{\theta}) }{ 1- \alpha}\right) - D \cdot  \left(\frac{\epsilon }{ (1-\alpha)^2} \right).
\end{equation}

{\bf Step 3: Identifying a satisficing action is hard.}\\
Our goal is to lower bound $\Prob( A_\tau \in S_{\theta})$. To do this, we consider two alternative infinite armed bandit models. Each induces an a probability measure as follows: \\
\begin{enumerate}
	\item The probability measure $\Prob(\cdot )$ corresponds to the infinite armed bandit model as described in the theorem statement. The collection $\theta \equiv (\theta_a)_{a \in \mathbb{N}}$ is drawn randomly according to the prior probabilities in \eqref{eq: lower bound construction}. The random seed $\xi$ is drawn uniformly from $[0,1]$. For each period $t$, the action $A_{t}=\psi(\hist, \xi)$ is prescribed by the policy $\psi$. Then $\Prob(R_{t} = 1 \mid A_{t}, \theta, \hist, \xi) = \theta_{A_t}$.  
	\item We consider an alternative model, which is identical except that rewards are always drawn from a Bernoulli distribution with mean $1/2-\Delta$. Precisely, we let $\mathbb{Q}$ be an alternative probability measure with the following properties: As before, the random seed $\xi$ is drawn uniformly from $[0,1]$, $A_{t} =\psi(\hist, \xi)$ for each period $t$, and $\theta$ is drawn from according to the prior probabilities in \eqref{eq: lower bound construction}. However, rewards are now independent from $\theta$, with $\mathbb{Q}(R_t =1 \mid A_t, \theta, \hist, \xi) = 1/2-\Delta$. \\
\end{enumerate}

The idea of this construction is that $D_{\rm KL}\left({\mathbb Q}(\hist = \cdot) \, || \, {\mathbb P}(\hist = \cdot)  \right)$ will reduce to considering the divergence in reward distributions, since this is the only source of discrepancy between the probability distributions. We continue to let $\tau \sim {\rm Geom}(1-\alpha)$ denote a geometric random variable that is mutually independent from $\theta$ and $(\Hc_{t})_{t\in \mathbb{N}}$ under both $\Prob$ and $\mathbb{Q}$. We take $\E_{\mathbb Q}[\cdot]$ to denote the expectation under the probability measure $\mathbb{Q}$. 

Under $\mathbb{Q}$, the algorithm's observations are independent of $\theta$, so 
\[
\mathbb{Q}(A_{t} \in S_{\theta}) =  \E_{\mathbb Q}\left[ \mathbb{Q}( A_{t} \in S_{\theta} \mid \theta) \right] = \E_{\mathbb Q}[ \delta ]  = \delta.
\]
Applying this gives,
\[
\mathbb{Q}(A_{\tau} \in S_{\theta}) =  \E_{\mathbb Q}\left[ \mathbb{Q}( A_{\tau} \in S_{\theta} \mid \tau) \right] = \delta.
\]
Then, by Pinsker's inequality 
\begin{align}\nonumber
\Prob(A_{\tau} \in S_{\theta})  & \leq \mathbb{Q}(A_\tau \in S_{\theta}) + \sqrt{\frac{1}{2} D_{\rm KL}\left(  \mathbb{Q}(A_{\tau} \in S_{\theta})  \, ||\, \Prob( A_{\tau} \in S_{\theta}) \right)} \\ 
&= \delta + \sqrt{\frac{1}{2} D_{\rm KL}\left(  \mathbb{Q}(A_{\tau} \in S_{\theta})  \, ||\, \Prob( A_{\tau} \in S_{\theta}) \right)}. \label{eq: bound from pinsker} 
\end{align}
We upper bound and expand the KL divergence through repeated use of the data-processing inequality and the chain rule. We have 
\begin{align*}
D_{\rm KL}\left(  \mathbb{Q}(A_\tau  = \cdot ) \, ||\, \Prob( A_{\tau} = \cdot \right) &\leq  D_{\rm KL}\left(  \mathbb{Q}(A_\tau  = \cdot ) \, ||\, \Prob( A_{\tau} = \cdot \right) \\
& \leq  D_{\rm KL}\left(  \mathbb{Q}( (\tau, \theta, \Hc_{\tau}, \xi ) = \cdot ) \, ||\, \Prob( (\tau, \theta, \Hc_{\tau}, \xi) = \cdot \right) \\ 
& =  D_{\rm KL}\left(  \mathbb{Q}( (\tau, \theta, \xi) = \cdot )\, ||\, \Prob( (\tau, \theta, \xi) = \cdot ) \right)+  D_{\rm KL}\left(  \mathbb{Q}(  \Hc_{\tau}  = \cdot  \mid \tau, \theta, \xi ) \, ||\, \Prob(\Hc_{\tau}  = \cdot  \mid \tau, \theta, \xi ) \right) \\
&= D_{\rm KL}\left(  \mathbb{Q}(  \Hc_{\tau}  = \cdot  \mid \tau, \theta, \xi ) \, ||\, \Prob(\Hc_{\tau}  = \cdot  \mid \tau, \theta, \xi) \right) \\
&=  \sum_{t=0}^{\infty} \mathbb{Q}(\tau =t)  D_{\rm KL}\left(  \mathbb{Q}(  \Hc_{\tau}  = \cdot  \mid  \tau = t, \theta, \xi ) \, ||\, \Prob(\Hc_{\tau}  = \cdot  \mid \tau=t, \theta, \xi ) \right) \\
&= \sum_{t=0}^{\infty} \mathbb{Q}(\tau =t)  D_{\rm KL}\left(  \mathbb{Q}(  \Hc_{t}  = \cdot  \mid  \theta, \xi ) \, ||\, \Prob(\Hc_{t}  = \cdot  \mid \theta, \xi ) \right)
\end{align*}
where the final equality uses the independence of $\tau$ and $(\theta, (\Hc_t)_{t\in \mathbb{N}}, \xi)$. Define the binary KL divergence function $d: [0,1]\times [0,1] \to \mathbb{R}$ by $d(p|| q) = p \log(p/q) + (1-p)\log((1-p)/ (1-q))$.  Now, by the chain rule and the relation $\Hc_t=(\Hc_{t-1}, A_t, R_t)$, we have 
\begin{align*}
D_{\rm KL}\left(  \mathbb{Q}(  \Hc_{t}  = \cdot  \mid  \theta, \xi ) \, ||\, \Prob(\Hc_{t}  = \cdot  \mid \theta, \xi ) \right) =& D_{\rm KL}\left(  \mathbb{Q}(  \Hc_{t-1}  = \cdot  \mid  \theta, \xi) \, ||\, \Prob(\Hc_{t-1}  = \cdot  \mid \theta, \xi) \right)\\
&+ D_{\rm KL}\left(  \mathbb{Q}(  A_t  = \cdot  \mid  \theta, \Hc_{t-1}, \xi) \,\, ||\,\, \Prob( A_t  = \cdot  \mid  \theta, \Hc_{t-1}, \xi) \right) \\
&+ D_{\rm KL}\left(  \mathbb{Q}(  R_t  = \cdot  \mid  \theta, \Hc_{t-1}, A_t, \xi) \, \, ||\,\, \Prob( R_t  = \cdot  \mid  \theta, \Hc_{t-1}, A_t, \xi) \right) \\
=& D_{\rm KL}\left(  \mathbb{Q}(  \Hc_{t-1}  = \cdot  \mid  \theta) \,\, ||\,\, \Prob(\Hc_{t-1}  = \cdot  \mid \theta) \right) +\E_{\mathbb{Q}}\left[ d\left( 1/2-D  \,\, || \,\, \theta_{A_t} \right) \right] \\
=& \cdots\\
=& \E_{\mathbb{Q}}  \left[ \sum_{\ell=0}^{t} d\left( 1/2-\Delta  \,\, || \,\, \theta_{A_\ell} \right) \right] \\
\leq & \E_{\mathbb{Q}}  \left[ \sum_{\ell=0}^{t} \left(\mathbf{1}(A_t \in S_{\theta}) d\left( 1/2-\Delta  \,\, || \,\, 1/2\right)+ \mathbf{1}(A_t \in {\rm Opt}_{\theta}  \right) d\left( 1/4  \,\, || \,\, 3/4\right) \right]   \\
=& \sum_{\ell=0}^{t} \mathbb{Q}(A_t \in S_{\theta}) d\left( 1/2-\Delta  \,\, || \,\, 1/2 \right) + \sum_{\ell=0}^{t} \mathbb{Q}(A_t \in {\rm OPt}_{\theta}) d\left( 1/2-\Delta  \,\, || \,\, 3/4 \right).
\end{align*}
Here we used that $D_{\rm KL}\left(  \mathbb{Q}(  A_t  = \cdot  \mid  \theta, \Hc_{t-1}, \xi) \,\, ||\,\, \Prob( A_t  = \cdot  \mid  \theta, \Hc_{t-1}, \xi) \right)=0$ since, conditioned on $(\Hc_{t-1}, \xi)$,  $A_t$ is almost surely equal to $\psi_{t}(\Hc_{t-1}, \xi)$ under $\mathbb{Q}(\cdot)$ or $\mathbb{P}(\cdot)$. The inequality uses that $1/2-\Delta \geq 1/4$ and $\theta_a \leq 3/4$ by hypothesis. Plugging this in above, we find
\begin{align*}
& D_{\rm KL}\left(  \mathbb{Q}(A_\tau  = \cdot ) \, ||\, \Prob( A_{\tau} = \cdot ) \right) \\
\leq&  \sum_{t=0}^{\infty} \mathbb{Q}(\tau =t) \sum_{\ell=0}^{t} \left[ \mathbb{Q}(A_\ell \in S_{\theta}) d\left( 1/2-\Delta  \,\, || \,\, 1/2 \right)+\mathbb{Q}(A_\ell \in {\rm OPt}_{\theta}) d\left( 1/4  \,\, || \,\, 3/4 \right)\right] \\
=& \sum_{t=0}^{\infty} \mathbb{Q}(\tau \geq t) \left[ \mathbb{Q}(A_t \in S_{\theta}) d\left( 1/2 - \Delta  \,\, || \,\, 1/2 \right)+\mathbb{Q}(A_t \in {\rm OPt}_{\theta}) d\left( 1/4 \,\, || \,\, 3/4 \right)\right]  \\
=& \sum_{t=0}^{\infty} \alpha^{t}  \delta d\left( 1/2-\Delta  \,\, || \,\, 1/2 \right) + \sum_{t=0}^{\infty} \alpha^{t}  \epsilon d\left( 1/4  \,\, || \,\, 3/4 \right)  \\
=& \frac{\delta \cdot \Delta \log(\frac{1+\Delta}{1-\Delta} )}{1-\alpha} + \frac{\epsilon \cdot (1/4) \log (\frac{5/4}{3/4})}{1-\alpha}\\
=& \frac{\delta \cdot \Delta \log(1+ \frac{2\Delta}{1-\Delta})}{1-\alpha}+ \frac{\epsilon \cdot (1/4) \log (\frac{5}{3})}{1-\alpha}\\
\leq& \frac{4\delta \cdot \Delta^2  }{1-\alpha}+ \frac{\epsilon/4}{1-\alpha}
\end{align*}
where the last step uses the requirement that $\frac{1}{1-\Delta}< 2$. To conclude, plugging the above into \eqref{eq: bound from pinsker} and using the concavity of the square root, we have shown 
\begin{equation}\label{eq: bound on prob of satisficing}
\Prob(A_{\tau} \in S_{\theta})   \leq \delta + \Delta \sqrt{ \frac{2\delta }{1-\alpha}} + \sqrt{\frac{\epsilon/4}{1-\alpha}}
\end{equation}

{\bf Step 4: Conclusion by plugging in for $D$ and $\epsilon$.} \\
Combining \eqref{eq: lower bound by prob of satisficing} and \eqref{eq: bound on prob of satisficing}, we have  

\begin{align*}
\sregret(\alpha, \psi, D) &\geq \left( \frac{ \Delta \cdot (1-\delta) - \Delta^2 \cdot \sqrt{ \frac{2\delta }{1-\alpha}}}{1-\alpha} \right)  - \sqrt{\epsilon} \cdot \frac{ \Delta/2 }{(1-\alpha)^{3/2}} - \epsilon \cdot  \left(\frac{D }{ (1-\alpha)^2} \right)\\
&\geq \left( \frac{ \Delta/2 - \Delta^2 \cdot \sqrt{ \frac{2\delta }{1-\alpha}}}{1-\alpha} \right)  - \sqrt{\epsilon} \cdot \frac{ 1/8 }{(1-\alpha)^{3/2}} -\epsilon \cdot  \left(\frac{1/4 }{ (1-\alpha)^2} \right) \\
&\geq \left( \frac{ \Delta/2 - \Delta^2 \cdot \sqrt{ \frac{2\delta }{1-\alpha}}}{1-\alpha} \right)  -  \cdot \frac{\sqrt{\epsilon} \cdot (3/8) }{(1-\alpha)^{2}} \\
&:= f(\Delta) - g(\epsilon )
\end{align*}
where in the final step we will require $\epsilon \leq 1$. 

We now focus on the first term, and will eventually pick $\epsilon$ so the remaining term is sufficiently small. Set 
\[
f(\Delta) =  \frac{1}{1-\alpha}  \cdot \left[  \frac{\Delta}{2}  - \Delta^2 \cdot \sqrt{ \frac{2\delta }{1-\alpha}} \right] 
\]
This is a quadratic function with global minimum at 
\[
\Delta^* = \argmin_{\Delta \in \mathbb{R} } f(\Delta)  =   \frac{1}{4\sqrt{2} } \cdot \sqrt{\frac{1-\alpha }{\delta}} 
\]

For 
\begin{equation}\label{eq: lower bound Delta choice}
\Delta_0 \equiv \min\left\{ \frac{1}{4} , \Delta^* \right\} =  \bigg\{\begin{array}{lr}
1/4 & \quad \text{for  }  2\delta \geq 1-\alpha \\
\Delta^* & \quad \text{for  }  2\delta \leq 1-\alpha
\end{array}
\end{equation}
we have 
\[ 
f(\Delta_0) = \left\{\begin{array}{lr}
\frac{1}{16} \cdot \frac{1}{1-\alpha} & \quad \text{for  }  \delta \geq 2(1-\alpha) \\
\frac{1}{16\sqrt{2} } \cdot \sqrt{\frac{1/\delta}{1-\alpha}} & \quad \text{for  }  \delta \leq 2(1-\alpha)  
\end{array}\right\} \geq \frac{1}{16} \cdot \min\left\{  \frac{1}{1-\alpha} \, , \,    \sqrt{\frac{1/2\delta}{1-\alpha}}  \right\}
\]
Now, we pick $\epsilon_0\leq 1$ so that 
\[ 
g(\epsilon_0) \leq f(\Delta)/2. 
\]
We need 
\[
\frac{\sqrt{\epsilon_0} \cdot (3/8) }{(1-\alpha)^{2}} \leq  \frac{1}{16} \cdot \min\left\{  \frac{1}{1-\alpha} \, , \,    \sqrt{\frac{1/2\delta}{1-\alpha}}  \right\}.
\]
This is satisfied with equality for
\begin{equation}\label{eq: lower bound epsilon choice}
\epsilon_0= \frac{1}{36}  \cdot \min\left\{ (1-\alpha)^2  \, , \,      \frac{(1-\alpha)^3}{2\delta}  \right\}
\end{equation}
This shows that for a choice of $\Delta= \Delta_0$ and $\epsilon \leq \epsilon_0$, as in the theorem statement, we have 
\[ 
\sregret(\alpha, \psi, D) \geq \frac{1}{32} \cdot \min\left\{  \frac{1}{1-\alpha} \, , \,    \sqrt{\frac{1/2\delta}{1-\alpha}}  \right\}.
\]
\end{proof}

\section{Closing Remarks}

We have put forth a way of thinking about satisficing in bandit learning.  
The per-period regret of an algorithm that learns a satisficing action $\tilde{A}$ does not converge to zero,
but the time required can often be far less than what it would be to learn an optimal action $A^*$.  Intuitively,
this advantage stems from the fact that the mutual information $I(\theta; \tilde{A})$, which can be thought of
as the number of bits of information about the model required to learn $\tilde{A}$, can be far less than $I(\theta; A^*)$.

Satisficing plays a particularly important role when there is a satisficing action $\tilde{A}$ for which $I(\theta; \tilde{A}) \ll I(\theta; A^*)$
and the agent exhibits time preference, valuing near-term over long-term rewards.  To express this
in terms of a formal objective, we considered expected discounted regret.  We also introduced satisficing Thompson sampling,
and established results pertaining to infinite-armed and linear bandits demonstrating that this variant of Thompson sampling
captures benefits of targeting a satisficing action. 

We believe this paper puts forth a useful conceptual framework and that satisficing Thompson sampling could serve as a useful design principle for time-sensitive learning problems. But, we also feel this paper takes only a very preliminary step toward understanding satisficing in modern bandit learning. Future work might try to analyze the rate distortion function and information ratio for broader classes of problems than addressed in this paper. We suspect models like the hierarchical bandit in Example 2 may be quite useful in practice, and warrant more complete study. Satisficing is even more important in reinforcement learning than in bandit problems. Most ideas in this paper extend gracefully to contextual bandits, but extensions to reinforcement seem difficult and are a fascinating direction for future work. 

\section*{Acknowledgments.}

The second author was generously supported by a research grant from Boeing and a Marketing Research Award from Adobe. A special thanks is owed to David Tse, who played an important role in the early stages of this work. It was David who first emphasized that bounds based on entropy can be vacuous and pointed us to references on rate-distortion theory. We also thank Tor Lattimore for thoughtful comments on an early draft of this work.

\appendix

	\section{Analysis of the Infinite-Armed Deterministic Bandit.}
\begin{thm}\label{thm: regret of TS}
	For all $\alpha \in [0,1]$, under Thompson sampling in the infinite-armed deterministic bandit,
	$$\E\left[\sum_{t=0}^\infty \alpha^t (R^* - \theta_{A_t})\right] =\frac{1}{2(1-\alpha)}.$$	
	Under satisficing Thompson sampling with tolerance $\epsilon = \sqrt{1-\alpha}$ in the infinite-armed deterministic bandit,
	$$\E\left[\sum_{t=0}^\infty \alpha^t (R^* - \theta_{A_t})\right] \leq \frac{1}{\sqrt{1-\alpha}}.$$
\end{thm}
\begin{proof}
In every period $t$, TS samples a new action $A_t \notin \{A_1,...,A_{t-1}\}$, which generates expected reward $\E[\theta_{A_t}] = \E[\theta_1] = 1/2$.  The optimal expected reward is 1, and therefore the expected discounted-regret of TS is
\[
\sum_{t=0}^{\infty} \alpha^{t} (1-1/2) = \frac{1}{2(1-\alpha)}.
\]
Now, let us analyze satisficing Thompson sampling. Let $\tau = \min\{t:\theta_{A_t} \geq 1-\epsilon\}$ denote the fist time it sampled a $\epsilon$--optimal action. Then,
\begin{eqnarray}\nonumber
\E\left[\sum_{t=0}^\infty \alpha^t (R^* - \theta_{A_t})\right]
&=& \E\left[\sum_{t=0}^\infty \alpha^t (1 - \theta_{A_t})\right] \\\nonumber 
&=& \E\left[\E\left[\sum_{t=0}^{\tau-1} \alpha^t (1-\theta_{A_t}) + \sum_{t=\tau}^\infty \alpha^t (1-\theta_{A_t}) \Big| \tau \right]\right] \\\nonumber
&=& \E\left[\E\left[\sum_{t=0}^{\tau-1} \alpha^t \E[1-\theta_{a} | \theta_a \leq 1-\epsilon]  + \sum_{t=\tau}^\infty \alpha^t \E[1-\theta_{a} | \theta_a \geq 1-\epsilon] \Big| \tau \right]\right] \\\nonumber
&=& \E\left[\E\left[\sum_{t=0}^{\tau-1} \alpha^t(1/2 + \epsilon/2)+ \sum_{t=\tau}^\infty \alpha^t(\epsilon/2) \Big| \tau \right]\right] \\\nonumber
&=& \E\left[\E\left[\sum_{t=0}^{\tau-1} \alpha^t(1/2) + \sum_{t=0}^\infty \alpha^t(\epsilon/2) \Big| \tau \right]\right] \\\nonumber
&=&\E\left[\frac{(1-\alpha^\tau)}{2(1-\alpha)} + \frac{ \epsilon}{2(1-\alpha)}\right] \\\label{eq: final bound for infinite armed deterministic bandit}
&\leq& \left(\frac{1}{2\epsilon} + \frac{\epsilon}{2(1-\alpha)}\right).
\end{eqnarray}
The inequality above follows from the calculation
$$\E[1 - \alpha^\tau] = 1 - \sum_{t=0}^\infty \epsilon (1-\epsilon)^t \alpha^t = 1 - \frac{\epsilon}{1 - \alpha(1-\epsilon)} = \frac{1-\alpha(1-\epsilon)-\epsilon}{1-\alpha(1-\epsilon)}\leq \frac{1-\alpha}{\epsilon}.$$
The final bound of $\sqrt{1/(1-\alpha)}$ follows by choosing the minimizer $\epsilon^* = \sqrt{1-\alpha}$ of equation \eqref{eq: final bound for infinite armed deterministic bandit}.
\end{proof}

%
%

\section{Proof of Lemma \ref{lem: geometric average of logs}.}
\label{se: geometric average of logs}

\begin{proof}
\begin{eqnarray*}
	\sum_{t=1}^{\infty} \gamma^{-t} \log(t)&\leq& \sum_{t=1}^{\infty} e^{-(1-\gamma)t} \log(t) \\
	&=& \sum_{t=2}^{\infty} e^{-(1-\gamma)t} \log(t) \\
	&\overset{*}{\leq}& \int_{1}^{\infty} e^{-(1-\gamma)x} \log(x+1)dx \\
	&=& \frac{1}{1-\gamma}\int_{1}^{\infty} e^{-u} \log\left(\frac{u}{1-\gamma}+1\right)du \\
	& \leq & \frac{1}{1-\gamma} \left( \left[ 1+ \log\left(\frac{1}{1-\gamma}\right)\right] \int_{1}^{\infty} e^{-u} du + \int_{1}^{\infty} e^{-u}\log(u) du \right) \\
	&=& \frac{1}{1-\gamma} \left( \left[ 1+ \log\left(\frac{1}{1-\gamma}\right)\right](1/e)+ \int_{1}^{\infty} e^{-u}\log(u) du \right) \\
	& \leq & \frac{1}{1-\gamma} \left[ 1+ \log\left(\frac{1}{1-\gamma}\right)\right]
\end{eqnarray*}
where the last step uses a numerical approximation to the indefinite integral
\[
\int_{1}^{\infty} e^{-u}\log(u) du \approx .22
\]
along with the fact that $1/e + .22 \approx .57< 1.$

The inequality (*) uses that for any $t\geq 2$
\[
e^{-(1-\gamma)t} \log(t) \leq \intop_{t-1}^{t} e^{-(1-\gamma)x} \log(x+1)
\]
since $e^{-(1-\gamma)x}$ is decreasing in $x$ and $\log(x)$ is increasing in $x$.
\end{proof}


\pagebreak

\bibliography{references}
\bibliographystyle{plainnat}

\vspace*{\fill}

\end{document}